\documentclass{article}
\usepackage[utf8]{inputenc}

\usepackage[final]{neurips_2019}
\usepackage{latexsym}
\usepackage{epsfig}
\usepackage{amsmath,amsthm,amssymb}
\usepackage[linesnumbered, ruled,vlined]{algorithm2e}
\usepackage{caption}
\usepackage{subcaption}
\usepackage{algorithmic}
\usepackage[most]{tcolorbox}
\usepackage{parskip}
\usepackage{color}
\usepackage{setspace}
\usepackage{mathtools}
\usepackage{fullpage}
\usepackage{hyperref}
\usepackage{bm}
\usepackage[]{algorithm2e}

\usepackage{soul}

\usepackage{wrapfig, blindtext}

\newcommand{\ignore}[1]{}

\newtheorem{theorem}{Theorem}

\newtheorem{proposition}[theorem]{Proposition}

\theoremstyle{plain}

\title{Wasserstein Robust Reinforcement Learning}

\author{
  Mohammed Amin Abdullah\footnotemark[1]\thanks{The first three authors are to be considered as joint first co-authors} \\
  Huawei R\&D UK \\
  \texttt{mohammed.abdullah@huawei.com } \\ \And 
  Hang Ren \footnotemark[1]\footnotemark[2] \\
  Huawei R\&D UK \\ Imperial College London \\
  \texttt{hang.ren@huawei.com} \\ \And 
  Haitham Bou-Ammar \footnotemark[1] \thanks{Honorary Lecturer Position at University College London.}\\
  Huawei R\&D UK \\  University College London \\
  \texttt{haitham.bouammar@huawei.com}   \\
  \AND 
  Vladimir Milenkovi\' c\footnotemark[2]\thanks{Work done while interning at the Reinforcement Learning Team in Huawei Technologies Research and Development in London.} \\
 Huawei R\&D UK \\ University of Cambridge \\
   \And
  Rui Luo \footnotemark[2]\\
  Huawei R\&D UK \\ University College London \\ 
  \And 
  Mingtian Zhang\footnotemark[2] \\
  Huawei R\&D UK \\ University College London \\ 
   \And
  Jun Wang \\
  Huawei R\&D UK \\ University College London \\ 
  }

\begin{document}

\maketitle

\begin{abstract}
  Reinforcement learning algorithms, though successful, tend to over-fit to training environments hampering their application to the real-world. This paper proposes $\text{W}\text{R}^{2}\text{L}$ -- a robust reinforcement learning algorithm with significant robust performance on low and high-dimensional control tasks. Our method formalises robust reinforcement learning as a novel min-max game with a Wasserstein constraint for a correct and convergent solver. Apart from the formulation, we also propose an efficient and scalable solver following a novel zero-order optimisation method that we believe can be useful to numerical optimisation in general. 
  We empirically demonstrate significant gains compared to standard and robust state-of-the-art algorithms on high-dimensional MuJuCo environments. 
\end{abstract}

\section{Introduction}

Reinforcement learning (RL) has become a standard tool for solving decision-making problems with minimal feedback. Applications with these characteristics are ubiquitous, including, but not limited to, computer games \citep{Silver}, robotics \citep{KoberPM2012, Deisenroth:2013:SPS:2688186.2688187, Ammar:2014:OML:3044805.3045027}, finance \citep{Fischer2018Reinforcement}, and personalised medicine \citep{make1010009}. Although significant progress has been made on developing algorithms for learning large-scale and high-dimensional reinforcement learning tasks, these algorithms often over-fit to training environments and fail to generalise across even slight variations of transition dynamics \citep{DBLP:journals/corr/abs-1810-12282, DBLP:journals/corr/abs-1902-07015}.   

Robustness to changes in transition dynamics, however, is a crucial component for adaptive and safe RL in real-world environments. To illustrate, consider a self-driving car scenario in which we attempt to design an agent capable of driving a vehicle smoothly, safely, and autonomously. A typical reinforcement learning work-flow to solving such a problem consists of building a simulator to emulate real-world scenarios, training in simulation, and then transferring resultant policies to physical systems for control. Unfortunately, such a strategy is easily prone to failure as designing accurate simulators that capture intricate complexities of large cities is extremely challenging. Rather than learning in simulation, another work-flow might consist of constructing a pipeline to directly learn on the hardware system itself. Apart from memory constraints, state-of-the-art reinforcement learning algorithms exhaust hundreds to millions of agent-environment interactions before acquiring successful behaviour. Of course, such high demands on sample complexities prohibit the direct application of learning algorithms on real-systems, leaving robustness to misspecified simulators a largely unresolved problem. 

Motivated by real-world applications, recent literature in reinforcement learning has focused on the above problems, proposing a plethora of algorithms for robust decision-making~\citep{Morimoto:2005:RRL:1119345.1119349, pmlr-v70-pinto17a, pmlr-v97-tessler19a}. Most of these techniques borrow from game theory to analyse, typically in a discrete state and actions spaces, worst-case deviations of agents' policies and/or environments, see \cite{RePEc:aea:aecrev:v:91:y:2001:i:2:p:60-66, nilim2005robust, iyengar2005robust, Namkoong:2016:SGM:3157096.3157344} and references therein. These methods have also been extended to linear function approximators~\cite[]{NIPS2015_6014}, and deep neural networks~\cite[]{DBLP:journals/corr/abs-1710-06537} showing (modest) improvements in performance gain across a variety of disturbances, e.g., action uncertainties, or dynamical model variations. 

Though successful in practice, current techniques to robust decision making remain task-specific, and there are  two major lingering drawbacks. First, these algorithms fail to provide a general robustness framework due to their specialised heuristics. In particular, algorithms designed for discrete state-action spaces fail to generalise to continuous domains and vice versa. This is due to their exploiting mathematical properties valid only for discrete state and action spaces, or for convex-concave functions, e.g., $\min \max f(\cdot) = \max \min f(\cdot)$. Clearly, such mathematical derivations are only loosely related to implementation, and further insights can be gathered with more rigorous attempts tackling the continuous problem itself. Apart from theoretical limitations, the second drawback of current approaches can be traced back to the process by which empirical validation is conducted. Rarely do the papers presenting these techniques compare gains against other robust algorithms in literature. In fact, focus is mainly diverted to outperforming state-of-the-art reinforcement learning -- a criterion easy to satisfy as algorithms from RL were never trained for robustness. Interestingly, upon careful evaluation, we came to realise that robust adversarial reinforcement learning of \cite{pmlr-v70-pinto17a}, for example, is superior to some of the more recently published works attempting to solve the same problem, see Section \ref{Sec:Exps}.

In this paper, we contribute to the above endeavour to solve the robustness problem in RL by proposing a generic framework for robust reinforcement learning that can cope with both discrete and continuous state and actions spaces. The algorithm we introduce, which we call \emph{Wasserstein Robust Reinforcement Learning} (WR$^2$L), Algorithm \ref{Algo:Main}, aims to find the best policy, where any given policy is judged by the worst-case dynamics amongst all candidate dynamics in a certain set. This set is essentially the average Wasserstein ball around a reference dynamics $\mathcal{P}_0$. The constraints makes the problem well-defined, as searching over arbitrary dynamics can only result in non-performing system. The measure of performance is the standard RL objective form, the expected return. Both the policy and the dynamics are parameterised; the policy parameters $\bm{\theta}_k$ may be the weights of a deep neural network, and the dynamics parameters $\bm{\phi}_j$ the parameters of a simulator or differential equation solver. The algorithm performs estimated descent steps in $\bm{\phi}$ space and - after (almost) convergence - performs an update of policy parameters, i.e., in $\bm{\theta}$ space. Since $\bm{\phi}_j$ may be high-dimensional, we adapt a zero'th order sampling method based extending \cite{salimans2017evolution} to make estimations of gradients, and in order to define the constraint set which $\bm{\phi}_j$ is bounded by, we generalise the technique to estimate Hessians (Proposition \ref{Hess_prop}). 

We emphasise that although access to a simulator with parameterisable dynamics are required, the actual reference dynamics $\mathcal{P}_0$ need not be known explicitly nor learnt by our algorithm. Put another way, we are in the ``RL setting'', not the ``MDP setting'' where the transition probability matrix is known \emph{a priori}. The difference is made obvious, for example, in the fact that we cannot perform dynamic programming, and the determination of a particular probability transition can only be estimated from sampling, not retrieved explicitly. Hence, our algorithm is not model-based in the traditional sense of learning a model to perform planning.

We believe our contribution is useful and novel for two main reasons. Firstly, our framing of the robust learning problem is in terms of dynamics uncertainty sets defined by Wasserstein distance. Whilst we are not the first to introduce the Wasserstein distance into the context of MDPs (see, e.g., \cite{yang2017convex} or \cite{lecarpentier2019non}), we believe our formulation is amongst the first suitable for application to the demanding application-space we desire, that being, high-dimensional, continuous state and action spaces.
Secondly, we believe our solution approach is both novel and effective (as evidenced by experiments below, see Section \ref{Sec:Exps}), and does not place a great demand on model or domain knowledge, merely access to a simulator or differentiable equation solver that allows for the parameterisation of dynamics. Furthermore, it is not computationally demanding, in particular, because it does not attempt to build a model of the dynamics, and operations involving matrices are efficiently executable using the Jacobian-vector product facility of automatic differentiation engines.

The rest of the paper is organised as follows. In the next section, we provide a background on notation and an overview of Wasserstein distance in its various forms. The problem formulation and main algorithm is then presented in Section \ref{WRRL_sec}.  In Section \ref{Sec:Implementation}, we describe the zero'th order method we use to estimate the Hessian matrix used by the algorithm, and the proof of its correctness is given.  In Section \ref{Sec:Related_work} we survey related literature. Experiments and results are given in Section \ref{Sec:Exps}, and finally, conclusions and future work are discussed in Section \ref{Sec:Conc_Ftr_work}.

\section{Background}
This section provides background material needed for the remainder of the paper. We first describe the reinforcement learning framework adopted in this paper, and then proceed to detail notions from Wasserstein distances needed to constrain our learning objective. 
\subsection{Reinforcement Learning}\label{Sec:RL}
In reinforcement learning, an agent interacts with an unknown and stochastic environment with the goal of maximising a notion of return \citep{Sutton:1998:IRL:551283, Jan, Busoniu:2010:RLD:1855346}. These problems are typically formalised as Markov decision processes (MDPs)\footnote{Please note that we present reinforcement learning with continuous states and actions. This allows us to easily draw similarities to optimal control as detailed later. Extending these notions to discrete settings is relatively straight-forward.} $\mathcal{M} = \left\langle \mathcal{S}, \mathcal{A}, \mathcal{P}, \mathcal{R}, \gamma \right\rangle$, where $\mathcal{S} \subseteq \mathbb{R}^{d}$ denotes the state space, $\mathcal{A} \subseteq \mathbb{R}^{n}$ the action space, $\mathcal{P}: \mathcal{S} \times \mathcal{A} \times \mathcal{S} \rightarrow [0, 1]$ 
is a state transition probability describing the system’s dynamics, $\mathcal{R}: \mathcal{S} \times \mathcal{A} \rightarrow \mathbb{R}$ is the reward function measuring the agent's performance, and $\gamma \in [0,1)$ specifies the degree to which rewards are discounted over time. 

At each time step $t$, the agent is in state $\bm{s}_{t} \in \mathcal{S}$ and must choose an action $\bm{a}_{t} \in \mathcal{A}$, transitioning it to a new state $\bm{s}_{t+1} \sim \mathcal{P}\left(\bm{s}_{t+1}|\bm{s}_{t}, \bm{a}_{t}\right)$, and yielding a reward $\mathcal{R}(\bm{s}_{t}, \bm{a}_{t})$. A policy $\pi: \mathcal{S} \times {A} \rightarrow [0,1]$ is defined as a probability distribution over state-action pairs, where $\pi(\bm{a}_{t}|\bm{s}_{t})$ represents the density of selecting action $\bm{a}_{t}$ in state $\bm{s}_{t}$. Upon consequent interactions with the environment, the agent collects a trajectory $\bm{\tau}$ of state-action pairs. The goal is to determine an optimal policy $\pi^{\star}$ by solving: 
\begin{equation}
\label{Eq:RLObjective}
    \pi^{\star} = \arg\max_{\pi} \mathbb{E}_{\bm{\tau} \sim p_{\pi}(\bm{\tau})}\left[\mathcal{R}_{\text{Total}}(\bm{\tau})\right],
\end{equation}
where $p_{\pi}(\bm{\tau})$ denotes the trajectory density function, and $\mathcal{R}_{\text{Total}}(\bm{\tau})$ the return, that is, the total accumulated reward: 
\begin{equation}
\label{Eq:Traj}
    p_{\pi}(\bm{\tau}) = \mu_{0}(\bm{s}_{0})\pi(\bm{a}_{0}|\bm{s}_{0})\prod_{t=1}^{T-1}\mathcal{P}(\bm{s}_{t+1}|\bm{s}_{t}, \bm{a}_{t})\pi(\bm{a}_{t}|\bm{s}_{t}) \ \ \text{and} \ \  \ \mathcal{R}_{\text{Total}}(\bm{\tau}) = \sum_{t=0}^{T-1}\gamma^{t}\mathcal{R}(\bm{s}_{t}, \bm{a}_{t}),
\end{equation}
with $\mu_{0}(\cdot)$ denoting the initial state distribution.

\subsection{Wasserstein Distance}
In our problem definition, we make use of a distance measure to bound allowed variations from a reference transition density $\mathcal{P}_{0} (\cdot)$. In general, a number of common metrics for measuring closeness between two probability distributions exist. Examples of which are total variation distance and Kullback-Leibler divergence. In this paper, however, we measure distance between two different dynamics by the \emph{Wasserstein distance}. This has a number of desirable properties; firstly, it is a genuine distance, exhibiting symmetry, which is a property that K-L divergence lacks. Secondly, it is very flexible in the forms of the distributions that can be compared; it can measure the distance between two discrete distributions, two continuous distributions, and a discrete and continuous distribution (this latter case implying another valuable advantage -  that the supports of the distributions can be different). In all cases, the Wasserstein distance is well-defined. Finally, and perhaps most importantly, the Wasserstein distance takes into account the underlying geometry of the space the distributions are defined on, which could be information that is fruitful to exploit in learning optimal control. Indeed this last point is the core motivator for our choosing Wasserstein distance for our algorithm, as shall be explained later. 

\paragraph{Definition:} Given a measurable space $(\mathcal{X}, \mathcal{F})$ with $\mathcal{X}$ being a metric space, a pair of discrete measures $\mu, \nu$ defined on this measurable space can be written as $\mu=\sum_{i=1}^n \mu_i\delta_{x_i}$ and $\nu=\sum_{j=1}^m\nu_j\delta_{y_j}$ where all $x_i, y_j \in \mathcal{X}$. A coupling $\kappa(\cdot, \cdot)$ of $\mu$ and $\nu$ is a measure over $\{x_1, \ldots, x_n\} \times \{y_1, \ldots y_m\}$ that preserves marginals, i.e, $\mu_i=\sum_j\kappa(\mu_i, \nu_j)$ $\forall i$ and $\nu_j = \sum_{i}\kappa(\mu_i, \nu_j)$ $\forall j$. This then induces a cost of ``moving'' the mass of $\mu$ to $\nu$, given as the (Frobenius) inner product $\langle \kappa, C \rangle$ where the matrix $C \in \mathbb{R}^{n \times m}$ has $[C]_{ij}=c_{ij}=d(x_i, y_j)$, i.e., the cost of moving a unit of measure from $x_i$ to $y_j$. Minimised over the space of all couplings $\mathbf{K}(\mu, \nu)$, we get the Wasserstein distance, also known as the \emph{Earth-Mover Distance} (EMD). 

More generally, let $\mathcal{X}$ be a metric space with metric $d(\cdot, \cdot)$. Let $\mathcal{C}(\mathcal{X})$ be the space of continuous functions on $\mathcal{X}$ and let $\mathcal{M}(\mathcal{X})$ be the set of probability measures on $\mathcal{X}$. Let $\mu, \nu \in \mathcal{M}(\mathcal{X})$. Let $\mathbf{K}(\mu, \nu)$ be the set of couplings between $\mu, \nu$:
\begin{equation}
\mathbf{K}(\mu, \nu) := \{\kappa \in \mathcal{M}(\mathcal{X}\times \mathcal{X})\, ; \, \forall (A, B) \subset \mathcal{X}\times \mathcal{X}, \kappa (A \times \mathcal{X}) = \mu(A), \kappa(\mathcal{X} \times B) = \nu(B)\}
\label{eq:}
\end{equation}

That is, the set of joint distributions $\kappa \in \mathcal{M}(\mathcal{X}\times \mathcal{X})$ whose marginals agree with $\mu$ and $\nu$ respectively.
Given a metric (serving as a cost function) $d(\cdot,\cdot)$ for $\mathcal{X}$, the $p$'th Wasserstein distance $W_p(\mu, \nu)$ for $p \geq 1$ between $\mu$ and $\nu$ is defined as:
\begin{equation}\label{eq::smoothed_wass}
W_p(\mu, \nu):= \left(\min_{\kappa \in \mathbf{K}(\mu, \nu)} \int_{\mathcal{X} \times \mathcal{Y}} d(x,y)^pd\kappa(x,y)\right)^{1/p}
\end{equation}

\section{Wasserstein Robust Reinforcement Learning}\label{WRRL_sec}

This section formalises robust reinforcement learning by equipping agents with capabilities of determining well-behaved policies under worst-case models which are bounded in $\epsilon$-Wasserstein balls. Our motivation for formalising $\text{W}\text{R}^{2}\text{L}$ is rooted in robust optimal control -- a field dedicated to determining optimal action-selection rules under uncertainties induced by modelling assumptions. Here an agent controlling a plant/system faces an adversary that optimises for a disturbance controller while aiming at minimising rewards received by the agent. Interestingly, these problems relate to two-player min-max games and provide a rich literature with efficient solutions in certain specific scenarios, e.g., discrete states and actions, robust linear quadratic regulators, among others. Though generic min-max objectives can lead to robustness (see \cite{pmlr-v70-pinto17a}), typical algorithms rooted in game-theory optimise well-behaved objectives by introducing additional structural assumptions to adversaries. For example, it is not uncommon in robust optimal control to assume the process by which disturbances are applied (e.g., additive, multiplicative), or to only consider a subset adhering to maximally bounded norms \footnote{Please note that this is not to say that robust optimal control is restricted to the above settings, see \cite{doyle2013feedback}.}.         

Starting from robust optimal control, we derive $\text{W}\text{R}^{2}\text{L}$'s objective by introducing two major refinements to standard reinforcement learning. Similar to robust control, we enable agents to perturb transition models from a reference simulator with the goal of determining worst-case scenarios. We do not, however, pose additional structural assumptions on the process by which adversaries apply these perturbations. Rather, we posit a parameterised class of disturbances and adopt zero'th order optimisation (see Section~\ref{Sec:Implementation}) for flexibility, thus broadening our application spectrum to high-dimensional and stochastic dynamical systems. Optimisation problems of this nature, on the other hand, tend to be ill-specified due to the unconstrained process by which models are fit. To bound allowed variations in transition models, and ensure correctness and tractability, we then introduce an $\epsilon$-ball Wasserstein constraint around a simulator $\mathcal{P}_{0}(\cdot)$ to guarantee convergence. 

Before continuing, however, it is worth revisiting the motivations for choosing such a distance. Per the definition, constraining the possible dynamics to be within an $\epsilon$-Wasserstein ball of a reference dynamics $\mathcal{P}_{0}(\cdot)$ means constraining it in a certain way. Wasserstein distance has the form $\text{mass} \times \text{distance}$. If this quantity is constrained to be less than a constant $\epsilon$, then if the mass is large, the distance is small, and if the distance is large, the mass is small. Intuitively, when modelling the dynamics of a system, it may be reasonable to concede that there could be a systemic error - or bias - in the model, but that bias should not be too large. It is also reasonable to suppose that occasionally, the behaviour of the system may be wildly different to the model, but that this should be a low-probability event. If the model is frequently wrong by a large amount, then there is no use in it. In a sense, the Wasserstein ball formalises these assumptions.

In what comes next, we detail the problem definition as a novel min-max game with Wasserstein constrained dynamics. In Section~\ref{Sec:Sol}, we elaborate a generic algorithm capable of robustly updating both model and policy parameters.   

\subsection{Problem Definition: Robust Objectives and Constraints}\label{Sec:ProbDefObj}
As mentioned earlier, the problem definition we introduce in this paper extends reinforcement learning in two directions. In the first, we introduce a min-max objective with parameterised transition models, while in the second, we incorporate Wasserstein constraints to bound allowed perturbations. 

\paragraph{Parameterising Policies and Transition Models:} Due to the continuous nature of the state and action spaces considered in this work, we resort to deep neural networks to parameterise policies, which we write as $\pi_{\bm{\theta}}(\bm{a}_{t}|\bm{s}_{t})$, where $\bm{\theta} \in \mathbb{R}^{d_{1}}$ is a set of tunable hyper-parameters to optimise. For instance, these policies can correspond to multi-layer perceptrons for MuJoCo environments, or to convolutional neural networks in case of high-dimensional states depicted as images. Exact policy details are ultimately application dependent and, consequently, provided in the relevant experiment sections. 

In principle, one can similarly parameterise transition models using deep networks (e.g., LSTM-type models) to provide one or more action-conditioned future state predictions. Though appealing, going down this path led us to agents that discover worst-case transition models which minimise training data but lack any valid physical meaning. For instance, original experiments we conducted on CartPole ended up involving transitions that alter angles without any change in angular velocities. More importantly, these effects became more apparent in high-dimensional settings where the number of potential minimisers increases significantly. It is worth noting that we are not the first to realise such an artifact when attempting to model physics-based dynamics using deep networks. Authors in~\citep{Jan2} remedy these problems by introducing Lagrangian mechanics to deep networks, while others~\citep{DBLP:journals/corr/abs-1803-08287, NIPS2018_7892} argue the need to model dynamics given by differential equation structures directly. 

Though incorporating physics-based priors to deep networks is an important and challenging task that holds the promise of scaling model-based reinforcement learning for efficient solvers, in this paper we rather study an alternative direction focusing on perturbing  differential equation solvers and/or simulators with respect to the dynamic specification parameters $\bm{\phi} \in \mathbb{R}^{d_{2}}$. Not only would such a consideration reduce the dimensionality of parameter spaces representing transition models, but would also guarantee valid dynamics due to the nature of the discrete differential equation solver. Though tackling some of the above problems, such a direction arrives with a new set of challenges related to computing gradients and Hessians of black-box solvers. In Section~\ref{Sec:Implementation}, we develop an efficient and scalable zero-order method for valid and accurate model updates. 

\paragraph{Unconstrained Loss Function:} Having equipped agents with the capability of representing policies and perturbing transition models, we are now ready to present an unconstrained version of $\text{W}\text{R}^{2}\text{L}$'s loss function. Borrowing from robust optimal control, we define robust reinforcement learning as an algorithm that learns best-case policies under worst-case transitions: 
\begin{equation}
\label{Eq:LossRobust}
    \max_{{\bm{\theta}}} \left[\min_{{\bm{\phi}}} \mathbb{E}_{\bm{\tau} \sim p_{\bm{\theta}}^{\bm{\phi}}(\bm{\tau})}\left[\mathcal{R}_{\text{total}}(\bm{\tau})\right]\right],
\end{equation}
where $p_{\bm{\theta}}^{\bm{\phi}}(\bm{\tau})$ is a trajectory density function parameterised by both policies and transition models, i.e., $\bm{\theta}$ and $\bm{\phi}$, respectively:
\begin{equation*}
    p_{\bm{\theta}}^{\bm{\phi}}(\bm{\tau}) = \mu_{0}(\bm{s}_{0})\pi(\bm{a}_{0}|\bm{s}_{0})\prod_{t=1}^{T-1}\underbrace{\mathcal{P}_{\bm{\phi}}(\bm{s}_{t+1}|\bm{s}_{t}, \bm{a}_{t})}_{\text{specs vector and diff. solver}}\underbrace{\pi_{\bm{\theta}}(\bm{a}_{t}|\bm{s}_{t})}_{\text{ \ \ deep network}}.
\end{equation*}
At this stage, it should be clear that our formulation, though inspired from robust optimal control, is, truthfully, more generic as it allows for parameterised classes of transition models without incorporating additional restrictions on the structure or the scope by which variations are executed\footnote{Of course, allowed perturbations are ultimately constrained by the hypothesis space. Even then, our model is more general compared to robust optimal control that assumes additive, multiplicative, or other forms of disturbances.}. 

\paragraph{Constraints \& Complete Problem Definition:}
Clearly, the problem in Equation \ref{Eq:LossRobust} is ill-defined due to the arbitrary class of parameterised transitions. To ensure well-behaved optimisation objectives, we next introduce constraints to bound search spaces and ensure convergence to feasible transition models. For a valid constraint set, our method assumes access to samples from a \emph{reference dynamics model} $\mathcal{P}_{0}(\cdot|\bm{s}, \bm{a})$, and bounds learnt transitions in an $\epsilon$-Wasserstein ball around $\mathcal{P}_{0}(\cdot|\bm{s}, \bm{a})$, i.e., the set defined as: 
\begin{equation}
\label{Eq:ConstraintSetOne}
    \mathcal{W}_{\epsilon} \left(\mathcal{P}_{\bm{\phi}}(\cdot),\mathcal{P}_{0}(\cdot)\right) = \left\{\mathcal{P}_{\bm{\phi}}(\cdot): \mathcal{W}_{2}^{2} \left(\mathcal{P}_{\bm{\phi}}(\cdot|\bm{s}, \bm{a}), \mathcal{P}_{0}(\cdot|\bm{s}, \bm{a})\right) \leq \epsilon, \ \ \forall (\bm{s}, \bm{a}) \in \mathcal{S} \times \mathcal{A} \right\},
\end{equation}
where $\epsilon \in \mathbb{R}_{+}$ is a hyperparameter used to specify the ``degree of robustness'' in a similar spirit to maximum norm bounds in robust optimal control. It is worth noting, that though we have access to samples from a reference simulator, our setting is by no means restricted to model-based reinforcement learning in an MDP setting. \ul{That is, our algorithm operates successfully given only traces from $\mathcal{P}_{0}$ accompanied with its specification parameters, e.g., pole lengths, torso masses, etc. -- a more flexible framework that does not require full model learners.}

Though defining a better behaved optimisation objective, the set in Equation \ref{Eq:ConstraintSetOne} introduces infinite number of constraints when considering continuous state and/or actions spaces. To remedy this problem, we target a relaxed version that considers average Wasserstein constraints instead: \begin{align}
\label{Eq:AverageConstraint}
\hat{\mathcal{W}}^{(\text{average})}_{\epsilon} \left(\mathcal{P}_{\bm{\phi}}(\cdot),\mathcal{P}_{0}(\cdot)\right) &= \left\{\mathcal{P}_{\bm{\phi}}(\cdot): \int_{(\bm{s}, \bm{a})} \mathcal{P}(\bm{s}, \bm{a})\mathcal{W}_{2}^{2} \left(\mathcal{P}_{\bm{\phi}}(\cdot|\bm{s}, \bm{a}), \mathcal{P}_{0}(\cdot|\bm{s}, \bm{a})\right) d(\bm{s}, \bm{a}) \leq \epsilon \right\} \\ \nonumber
& = \left\{\mathcal{P}_{\bm{\phi}}(\cdot): \mathbb{E}_{(\bm{s}, \bm{a}) \sim \mathcal{P}}\left[\mathcal{W}_{2}^{2} \left(\mathcal{P}_{\bm{\phi}}(\cdot|\bm{s}, \bm{a}), \mathcal{P}_{0}(\cdot|\bm{s}, \bm{a})\right)\right] \leq \epsilon \right\}
\end{align}
In words, Equation \ref{Eq:AverageConstraint} defines a set where \emph{expected} Wasserstein distance is bounded by $\epsilon$. Expectation in the above equation is evaluated over state-action pairs sampled according to $\mathcal{P}(\bm{s}, \bm{a})$ which is policy and transition-model dependent. Precisely, one can factor $\mathcal{P}(\bm{s}, \bm{a})$ as: 
\begin{align*}
    \mathcal{P}(\bm{s} \in \mathcal{S}, \bm{a} \in \mathcal{A}) = \mathcal{P}(\bm{a} \in \mathcal{A}|\bm{s} \in \mathcal{S})\mathcal{P}(\bm{s}\in \mathcal{S}) = \pi(\bm{a} \in \mathcal{A}|\bm{s} \in \mathcal{S})\rho_{\pi}^{\bm{\phi}_{0}}(\bm{s} \in \mathcal{S}),  
\end{align*}
where $\bm{s} \in \mathcal{S}$ and $\bm{a} \in \mathcal{A}$ are to be interpreted as events being elements of the state and action sets  -- a notation better suites for the continuous nature of the considered random variables. Moreover,  $\rho_{\pi}^{\bm{\phi}_{0}}(\bm{s} \in \mathcal{S})$ is a uniform distribution over state-actions pairs sampled from a trajectory. Precisely, the way we compute the expected Wasserstein distance is two steps. In the first, given a batch of trajectories sampled according to any policy $\pi$, potentially of varying lengths, we create a bucket of state-action pairs\footnote{Since we do not require access to current policies in order to compute the average Wasserstein distance, an argument can be made that $\text{W}\text{R}^{2}\text{L}$ support off-policy constraint evaluation. This is an important link that we plan to further exploit in the future to improve efficiency, scalablity, and enable transfer between various tasks. }. Given such data, we then compute expected Wasserstein distance over the pairs using Monte-Carlo estimation. The $\pi$ we use is one which samples actions uniformly at random.

With this in mind, we arrive at $\text{W}\text{R}^{2}\text{L}$'s optimisation problem allowing for best policies under worst-case yet bounded transition models: 

\begin{tcolorbox}[enhanced, arc=0pt,outer arc=0pt, colback=white, colframe=black, drop shadow={black,opacity=1}]
\textbf{\underline{Wasserstein Robust Reinforcement Learning Objective:}}
\begin{align}
\label{Eq:FinalObjective}
    &\max_{{\bm{\theta}}} \left[\min_{{\bm{\phi}}} \mathbb{E}_{\bm{\tau} \sim p_{\bm{\theta}}^{\bm{\phi}}(\bm{\tau})}\left[\mathcal{R}_{\text{total}}(\bm{\tau})\right]\right] \ \ \text{s.t.} \ \ \mathbb{E}_{(\bm{s}, \bm{a})\sim \pi(\cdot)\rho_{\pi}^{\bm{\phi}_{0}}(\cdot)}\left[\mathcal{W}_{2}^{2} \left(\mathcal{P}_{\bm{\phi}}(\cdot|\bm{s}, \bm{a}), \mathcal{P}_{0}(\cdot|\bm{s}, \bm{a})\right)\right] \leq \epsilon
\end{align}
\end{tcolorbox}

\subsection{Solution Methodology}\label{Sec:Sol}
Having derived our formal problem definition, this section presents our approach to solving for $\bm{\theta}$ and $\bm{\phi}$ in the objective of Equation \ref{Eq:FinalObjective}. On a high level, our solution methodology follows an alternating procedure interchangeably updating one variable given the other fixed. 

\paragraph{Updating Policy Parameters:} It is clear from Equation \ref{Eq:FinalObjective} that the average Wasserstein distance constraint is independent from $\bm{\theta}$ and can, in fact, use any other policy $\pi$ to estimate the expectation. Hence, given a fixed set of model parameters, $\bm{\theta}$ can be updated by solving the relevant sub-problem of Equation \ref{Eq:FinalObjective} written as: 
\begin{equation*}
    \max_{{\bm{\theta}}} \mathbb{E}_{\bm{\tau} \sim p_{\bm{\theta}}^{\phi}(\bm{\tau})}\left[\mathcal{R}_{\text{total}}(\bm{\tau})\right]. 
\end{equation*}
Interestingly, this problem is a standard reinforcement learning one with a crucial difference in that traces are sampled according to the transition model given by fixed model parameters, $\bm{\phi}$ that ultimately differ from these of the original simulator $\bm{\phi}_{0}$. Consequently, one can easily adapt any policy search method for updating policies under fixed dynamical models. As described later in Section \ref{Sec:Implementation}, we make use of proximal policy optimisation \citep{schulman2017proximal}, for instance, to update such action-selection-rules.

\paragraph{Updating Model Parameters:} Now, we turn our attention to solving the average constraint optimisation problem needed for updating $\bm{\phi}$ given a set of fixed policy parameters $\bm{\theta}$. Contrary to the previous step, here, the Wasserstein constraints play an important role due to their dependence on $\bm{\phi}$. Unfortunately, even with the simplification introduced in Section \ref{Sec:ProbDefObj} the resultant constraint is still difficult to computer in general, the difficulty being the evaluation of the Wasserstein term\footnote{The situation is easier in case of two Gaussian densities. Here, however, we keep the treatment general by proposing an alternative direction using Taylor expansions.}. 

To alleviate this problem, we propose to approximate the constraint in \eqref{Eq:FinalObjective} by its Taylor expansion up to second order. That is, defining 
\begin{equation*}
    W(\bm{\phi}) :=  \mathbb{E}_{(\bm{s}, \bm{a})\sim \pi(\cdot)\rho_{\pi}^{\bm{\phi}_{0}}(\cdot)}\left[\mathcal{W}_{2}^{2} \left(\mathcal{P}_{\bm{\phi}}(\cdot|\bm{s}, \bm{a}), \mathcal{P}_{0}(\cdot|\bm{s}, \bm{a})\right)\right] 
\end{equation*}

The above can be approximated around $\bm{\phi}_{0}$ by a second-order Taylor as: 
\begin{equation*}
     W(\bm{\phi})  \approx  W(\bm{\phi}_{0}) + \nabla_{\bm{\phi}} W(\bm{\phi}_{0})^{\mathsf{T}}\left(\bm{\phi}-\bm{\phi}_{0}\right) + \frac{1}{2}\left(\bm{\phi} - \bm{\phi}_{0}\right)^{\mathsf{T}}\nabla^{2}_{\bm{\phi}} W(\bm{\phi}_{0})\left(\bm{\phi} - \bm{\phi}_{0}\right).
\end{equation*}

Recognising that $W(\bm{\phi}_{0}) = 0$ (the distance between the same probability densities), and $\nabla_{\bm{\phi}} W(\bm{\phi}_{0}) = 0$ since $\bm{\phi}_{0}$ minimises  $W(\bm{\phi})$, we can simplify the Hessian approximation by writing: 
\begin{equation*}
 W(\bm{\phi}) \approx \frac{1}{2} (\bm{\phi} - \bm{\phi}_{0})^{\mathsf{T}} \nabla^{2}_{\bm{\phi}} W(\bm{\phi}_{0})(\bm{\phi} - \bm{\phi}_{0}).
\end{equation*}

Substituting our approximation back in the original problem in Equation \ref{Eq:FinalObjective}, we reach the following optimisation problem for determining model parameter given fixed policies:
\begin{equation}
\label{Eq:UpdatingModel}
    \min_{\bm{\phi}} \mathbb{E}_{\bm{\tau} \sim p_{\bm{\theta}}^{\bm{\phi}}(\bm{\tau})}\left[\mathcal{R}_{\text{total}}(\bm{\tau})\right] \ \ \text{s.t.} \ \ \frac{1}{2} (\bm{\phi} - \bm{\phi}_{0})^{\mathsf{T}} \bm{H}_{0}(\bm{\phi} - \bm{\phi}_{0}) \leq \epsilon, 
\end{equation} 
where $\bm{H}_{0} =  \nabla^{2}_{\bm{\phi}}\, \mathbb{E}_{(\bm{s}, \bm{a})\sim \pi(\cdot)\rho_{\pi}^{\bm{\phi}_{0}}(\cdot)}\left[\mathcal{W}_{2}^{2} \left(\mathcal{P}_{\bm{\phi}}(\cdot|\bm{s}, \bm{a}), \mathcal{P}_{0}(\cdot|\bm{s}, \bm{a})\right)\right]\bigg|_{\bm{\phi}=\bm{\phi}_0}$ is the Hessian of the expected squared $2$-Wasserstein distance evaluated at $\bm{\phi}_0$.

Optimisation problems with quadratic constraints can be efficiently solved using interior-point methods. To do so, one typically approximates the loss with a first-order expansion and determines a closed-form solution. Consider a pair of parameters $\bm{\theta}^{[k]}$ and $\bm{\phi}^{[j]}$ (which will correspond to parameters of the $j$'th inner loop of the $k$'th outer loop in the algorithm we present). To find $\bm{\phi}^{[j+1]}$, we solve:
\begin{equation*}
    \min_{\bm{\phi}} \nabla_{\bm{\phi}} \mathbb{E}_{\bm{\tau}\sim p_{\bm{\theta}}^{\bm{\phi}}(\bm{\tau})}\left[\mathcal{R}_{\text{total}}(\bm{\tau})\right] \Bigg|^{\mathsf{T}}_{\bm{\theta}^{[k]}, \bm{\phi}^{[j]}}(\bm{\phi} - \bm{\phi}^{[j]}) \ \ \text{s.t.} \ \ \frac{1}{2} (\bm{\phi} - \bm{\phi}_{0})^{\mathsf{T}}\bm{H}_{0}(\bm{\phi} - \bm{\phi}_{0})\leq \epsilon.
\end{equation*}
It is easy to show that a minimiser to the above equation can derived in a closed-form as: 
\begin{align}
\label{Eq:UpdateEqPhi}
    \bm{\phi}^{[j+1]} = \bm{\phi}_{0} - \sqrt{\frac{2\epsilon}{\bm{g}^{[k,j] \mathsf{T}}\bm{H}_{0}^{-1}\bm{g}^{[k,j]}}} \bm{H}_{0}^{-1}\bm{g}^{[k,j]},
\end{align}
with $\bm{g}^{[k,j]}$ denoting the gradient\footnote{\textbf{Remark:} Although this looks superficially similar to an approximation made in TRPO \cite[]{pmlr-v37-schulman15}, the latter aims to optimise the \emph{policy parameter} rather than dynamics. Furthermore, the constraint is based on the Kullback-Leibler divergence rather than the Wasserstein distance} evaluated at $\bm{\theta}^{[k]}$ and $\bm{\phi}^{[j]}$, i.e., $\bm{g}^{[k,j]} = \nabla_{\bm{\phi}}\mathbb{E}_{\bm{\tau} \sim p_{\bm{\theta}}^{\bm{\phi}}(\bm{\tau})} \mathbb{E}\left[\mathcal{R}_{\text{total}}(\bm{\tau})\right]\Bigg|_{\bm{\theta}^{[k]}, \bm{\phi}^{[j]}}$.

\paragraph{Generic Algorithm:} Having described the two main steps needed for updating policies and models, we now summarise these findings in the pseudo-code in Algorithm~\ref{Algo:Main}. 
\begin{algorithm}[h!]
\caption{Wasserstein Robust Reinforcement Learning}
\label{Algo:Main}
\begin{algorithmic}[1]
 \STATE \textbf{Inputs:} Wasserstein distance Hessian, $\bm{H}_{0}$ evaluated at $\bm{\phi}_{0}$ under any policy $\pi$, radius of the Wasserstein ball $\epsilon$, and the reference simulator specification parameters $\bm{\phi}_{0}$
 \STATE Initialise $\bm{\phi}^{[0]}$ with $\bm{\phi}_0$ and  policy parameters $\bm{\theta}^{[0]}$ arbitrarily 
 \FOR{$k = 0, 1, \dots $}
 \STATE $\bm{x}^{[0]} \leftarrow \bm{\phi}^{[0]}$ 
 \STATE $j \leftarrow 0$
    \STATE \underline{\textbf{Phase I: Update model parameter while fixing the policy:}}
    \WHILE{termination condition not met} 
        \STATE Compute descent direction for the model parameters as given by Equation~\ref{Eq:UpdateEqPhi}: $$\bm{p}^{[j]} \leftarrow \bm{\phi}_{0} - \sqrt{\frac{2\epsilon}{\bm{g}^{[k,j] \mathsf{T}}\bm{H}_{0}^{-1}\bm{g}^{[k,j]}}} \bm{H}_{0}^{-1}\bm{g}^{[k,j]} - \bm{x}^{[j]}$$ \label{alg_p_comp}
        \STATE Update candidate solution, while satisfying step size conditions (see discussion below)  on the learning rate $\alpha$: 
        $$ \bm{x}^{[j+1]} \leftarrow \bm{x}^{[j]} + \alpha \bm{p}^{[j]}$$ \label{alg_x_update}
        \STATE $j \leftarrow j+1$   
    \ENDWHILE
    \STATE Perform model update setting $\bm{\phi}^{[k+1]} \leftarrow \bm{x}^{[j]}$ \label{alg_phi_update}
    \STATE \underline{\textbf{Phase II: Update policy given new model parameters:}}
    \STATE Use any standard reinforcement learning algorithm for \emph{ascending} in the gradient direction, e.g., $\bm{\theta}^{[k+1]} \leftarrow \bm{\theta}^{[k]} + \beta^{[k]}\nabla_{\bm{\theta}}\mathbb{E}_{\bm{\tau}\sim p_{\bm{\theta}}^{\bm{\phi}}(\bm{\tau})}\left[\mathcal{R}_{\text{total}}(\bm{\tau})\right]\Bigg|_{\bm{\theta}^{[k]}, \bm{\phi}^{[k+1]}}$, with $\beta^{[k]}$ is a policy learning rate. \label{alg_theta_update}
 \ENDFOR
\end{algorithmic}
\end{algorithm}
As the Hessian\footnote{Please note our algorithm does not need to store the Hessian matrix. In line 7 of the algorithm, it is clear that we require Hessian-vector products. These can be easily computed using computational graphs without having access to the full Hessian matrix.} of the Wasserstein distance is evaluated based on reference dynamics and any policy $\pi$, we pass it, along with $\epsilon$ and $\bm{\phi}_{0}$ as inputs. Then Algorithms~\ref{Algo:Main} operates in a descent-ascent fashion in two main phases. In the first, lines 5 to 10 in Algorithm~\ref{Algo:Main}, dynamics parameters are updated using the closed-form solution in Equation~\ref{Eq:UpdateEqPhi}, while ensuring that learning rates abide by a step size condition (we used the Wolfe conditions~\citep{wolfe1969convergence}, though it can be some other method). With this, the second phase (line 11) utilises any state-of-the-art reinforcement learning method to adapt policy parameters generating $\bm{\theta}^{[k+1]}$.  

Regarding the termination condition for the inner loop, we leave this as a decision for the user. It could be, for example, a large finite time-out, or the norm of the gradient  $\bm{g}^{[k,j]}$ being below a threshold, or whichever happens first. 

\section{Zero'th Order Wasserstein Robust Reinforcement Learning}\label{Sec:Implementation}
So far, we have presented an algorithm for robust reinforcement learning assuming accessibility to first and second-order information of the loss and constraint. It is relatively easy to attain such information if we were to follow a model-based setting that parameterises transitions with deep networks. As mentioned earlier, however, we follow another route that utilises a black-box optimisation scheme as deep neural networks are not always suitable for dynamical systems grounded in Lagrangian mechanics and physics \citep{Jan2, IROS_2019_DeLaN_Energy_Control}.   

This section details our zero-order robust solver, where we present an implementation of our approach for a scenario where training can be done on a simulator for which the dynamics are parameterised and can be altered at will. Whilst we refer to simulators (since much of RL training is done on such), almost exactly the same applies to differential equation solvers or other software based techniques for training a policy before deployment into the real world. 

To elaborate, consider a simulator $\mathbb{S}_{\bm{\phi}}$ for which the dynamics are parameterised by a real vector $\bm{\phi}$, and for which we can execute steps of a trajectory (i.e., the simulator takes as input an action $\bm{a}$ and gives back a successor state and reward). For generating novel physics-grounded transitions, one can simply alter $\bm{\phi}$ and execute the instruction in $\mathbb{S}_{\bm{\phi}}$ from some a state $\bm{s} \in \mathcal{S}$, while applying an action $\bm{a} \in \mathcal{A}$. Not only does this ensure valid (under mechanics) transitions, but also promises scalability as specification parameters typically reside in lower dimensional spaces compared to the number of tuneable weights when using deep networks as transition models.

As we do not explicitly model transitions (e.g., the intricate operations of the simulator or differential equations solver), one has to tackle an additional challenge when requiring gradient or Hessian information to perform optimisation. Namely, if the idea of parameterising simulators through dynamic specifications in Phases I and II of Algorithm \ref{Algo:Main} is to be successfully executed, we require a procedure for estimating first and second-order information based on only function value evaluations of $\mathbb{S}_{\bm{\phi}}$.     

Next, we elaborate how one can acquire such estimates by proposing a novel zero-order method for estimating gradients and Hessians that we  use in our experiments that demonstrate scalability, and robustness on high-dimensional robotic environments.  

\paragraph{Gradient Estimation:} Recalling the update rule in Phase I of Algorithm \ref{Algo:Main}, we realise the need for, estimating the gradient of the loss function with respect to the vector specifying the dynamics of the environment, i.e., $\bm{g}^{[k,j]} = \nabla_{\bm{\phi}}\mathbb{E}_{\bm{\tau} \sim p_{\bm{\theta}}^{\bm{\phi}}(\bm{\tau})} \left[\mathcal{R}_{\text{total}}(\bm{\tau})\right]\Bigg|_{\bm{\theta}^{[k]}, \bm{\phi}^{[j]}}$ at each iteration of the inner-loop $j$. Handling simulators as black-box models, we estimate the gradients by sampling from a Gaussian distribution with mean $\bm{0}$ and ${\sigma}^{2}\bm{I}$ co-variance matrix. Our choice for such estimates is not arbitrary but rather theoretically grounded as one can easily prove the following proposition:
\begin{proposition}[Zero-Order Gradient Estimate]
For a fixed $\bm{\theta}$ and $\bm{\phi}$, the gradient can be computed as: 
\begin{equation*}
    \nabla_{\bm{\phi}}\mathbb{E}_{\bm{\tau} \sim p_{\bm{\theta}}^{\bm{\phi}}(\bm{\tau})} \left[\mathcal{R}_{\text{total}}(\bm{\tau})\right] = \frac{1}{\sigma^{2}} \mathbb{E}_{\bm{\xi} \sim \mathcal{N}(\bm{0}, \sigma^{2}\bm{I})}\left[\bm{\xi}\int_{\bm{\tau}} p_{\bm{\theta}}^{\bm{\phi} + \bm{\xi}}(\bm{\tau})\mathcal{R}_{\text{total}}(\bm{\tau})d\bm{\tau}\right].
\end{equation*}
\end{proposition}
\begin{proof}
The proof of the above proposition can easily be derived by combining the lines of reasoning in \citep{salimans2017evolution, RePEc:cor:louvco:2011001}, while extending to the parameterisation of dynamical models. To commence, begin by defining $\mathcal{J}_{\bm{\theta}}(\bm{\phi}) = \mathbb{E}_{\bm{\tau} \sim p_{\bm{\theta}}^{\phi}}\left[\mathcal{R}_{\text{total}}(\bm{\tau})\right]$ for fixed policy parameters $\bm{\theta}$. Given any perturbation vector, $\bm{\xi} \sim \mathcal{N}(\bm{0}, \sigma^{2}\bm{I})$, we can derive (through a Taylor expansion) the following: 
\begin{align*}
    \mathcal{J}_{\bm{\theta}}(\bm{\phi} + \bm{\xi}) = \mathcal{J}_{\bm{\theta}}(\bm{\phi}) + \bm{\xi}^{\mathsf{T}}\nabla_{\bm{\phi}}\mathcal{J}_{\bm{\theta}}(\bm{\phi}) + \frac{1}{2}\bm{\xi}^{\mathsf{T}}\nabla^{2}_{\bm{\phi}}\mathcal{J}_{\bm{\theta}}(\bm{\phi})\bm{\xi} + \mathcal{O}\left(\text{higher-order terms}\right).
\end{align*}
Multiplying by $\bm{\xi}$, and taking the expectation on both sides of the above equation, we get: 
\begin{align*}
    \mathbb{E}_{\bm{\xi} \sim \mathcal{N}(0, \sigma^{2}\bm{I})}\left[\bm{\xi}\mathcal{J}_{\bm{\theta}}(\bm{\phi} + \bm{\xi})\right] &= \mathbb{E}_{\bm{\xi}\sim \mathcal{N}(\bm{0}, \sigma^{2}\bm{I})}\left[\bm{\xi}\mathcal{J}_{\bm{\theta}}(\bm{\phi})+\bm{\xi}\bm{\xi}^{\mathsf{T}}\nabla_{\bm{\phi}}\mathcal{J}_{\bm{\theta}}(\bm{\phi})+\frac{1}{2}\bm{\xi}\bm{\xi}^{\mathsf{T}}\nabla_{\bm{\phi}}^{2}\mathcal{J}_{\bm{\theta}}(\bm{\phi})\bm{\xi}\right] \\
    & = \sigma^{2} \nabla_{\bm{\phi}} \mathcal{J}_{\bm{\theta}}(\bm{\phi}).
\end{align*}
Dividing by $\sigma^{2}$, we derive the statement of the proposition as: 
\begin{align*}
    \nabla_{\bm{\phi}} \mathcal{J}_{\bm{\theta}}(\bm{\phi}) = \frac{1}{\sigma^{2}} \mathbb{E}_{\bm{\xi} \sim \mathcal{N}(0, \sigma^{2}\bm{I})}\left[\bm{\xi}\mathcal{J}_{\bm{\theta}}(\bm{\phi} + \bm{\xi})\right]
\end{align*}
\end{proof}

\paragraph{Hessian Estimation:} Having derived a zero-order gradient estimator, we now generalise these notions to a form allowing us to estimate the Hessian. It is also worth reminding the reader that such a Hessian estimator needs to be performed one time only before executing the instructions in Algorithm \ref{Algo:Main} (i.e., $\bm{H}_{0}$ is passed as an input). Precisely, we prove the following proposition: \\

\begin{proposition}[Zero-Order Hessian Estimate]\label{Hess_prop}
The hessian of the Wasserstein distance around $\bm{\phi}_{0}$ can be estimated based on function evaluations. Recalling that $\bm{H}_{0} =  \nabla^{2}_{\bm{\phi}}\, \mathbb{E}_{(\bm{s}, \bm{a})\sim \pi(\cdot)\rho_{\pi}^{\bm{\phi}_{0}}(\cdot)}\left[\mathcal{W}_{2}^{2} \left(\mathcal{P}_{\bm{\phi}}(\cdot|\bm{s}, \bm{a}), \mathcal{P}_{0}(\cdot|\bm{s}, \bm{a})\right)\right]\bigg|_{\bm{\phi}=\bm{\phi}_0}$, and defining  $\mathcal{W}_{(\bm{s}, \bm{a})}(\bm{\phi}):=\mathcal{W}_{2}^{2} \left(\mathcal{P}_{\bm{\phi}}(\cdot|\bm{s}, \bm{a}), \mathcal{P}_{0}(\cdot|\bm{s}, \bm{a})\right)$, we prove: 
\begin{align*}
   \bm{H}_{0} &= \frac{1}{\sigma^{2}} \mathbb{E}_{\bm{\xi}\sim \mathcal{N}(\bm{0}, \sigma^{2}\bm{I})}\Bigg[\frac{1}{\sigma^{2}}\bm{\xi}\left(\mathbb{E}_{(\bm{s}, \bm{a})\sim \pi(\cdot)\rho_{\pi}^{\bm{\phi}_{0}}(\cdot)}\left[\mathcal{W}_{(\bm{s}, \bm{a})}\left(\bm{\phi}_{0} + \bm{\xi}\right)\right]\right)\bm{\xi}^{\mathsf{T}} \\
    &\hspace{20em} - \mathbb{E}_{(\bm{s}, \bm{a})\sim \pi(\cdot)\rho_{\pi}^{\bm{\phi}_{0}}(\cdot)}\left[\mathcal{W}_{(\bm{s}, \bm{a})}(\bm{\phi}_{0} + \bm{\xi})\right]\bm{I}\Bigg].    
\end{align*}
\end{proposition}
\begin{proof}
Commencing with the right-hand-side of the above equation, we perform second-order Taylor expansions for each of the two terms under under the expectation of $\bm{\xi}$. Namely, we write: 
\begin{align}
\label{Eq:HessianApprox}
  \bm{H}_{0} &= \frac{1}{\sigma^{2}} \mathbb{E}_{\bm{\xi}\sim \mathcal{N}(\bm{0}, \sigma^{2}\bm{I})}\Bigg[\frac{1}{\sigma^{2}}\bm{\xi}\left(\mathbb{E}_{(\bm{s}, \bm{a})\sim \pi(\cdot)\rho_{\pi}^{\bm{\phi}_{0}}(\cdot)}\left[\mathcal{W}_{(\bm{s}, \bm{a})}\left(\bm{\phi}_{0} + \bm{\xi}\right)\right]\right)\bm{\xi}^{\mathsf{T}} \\\nonumber
    &\hspace{23em} - \mathbb{E}_{(\bm{s}, \bm{a})\sim \pi(\cdot)\rho_{\pi}^{\bm{\phi}_{0}}(\cdot)}\left[\mathcal{W}_{(\bm{s}, \bm{a})}(\bm{\phi}_{0} + \bm{\xi})\right]\bm{I}\Bigg] \\ \nonumber
    & \approx \frac{1}{\sigma^{4}}\mathbb{E}_{\bm{\xi}\sim\mathcal{N}(\bm{0}, \sigma^{2}\bm{I})}\Bigg[\mathbb{E}_{(\bm{s}, \bm{a})\sim \pi(\cdot)\rho_{\pi}^{\phi_{0}}}\left[\mathcal{W}_{(\bm{s},\bm{a})}(\bm{\phi}_{0})\right]\bm{\xi}\bm{\xi}^{\mathsf{T}}  + \bm{\xi}\bm{\xi}^{\mathsf{T}}\nabla_{\bm{\phi}}\mathbb{E}_{(\bm{s}, \bm{a})\sim \pi(\cdot)\rho_{\pi}^{\phi_{0}}}\left[\mathcal{W}_{(\bm{s},\bm{a})}(\bm{\phi}_{0})\right]\bm{\xi}^{\mathsf{T}} \\ \nonumber
   & \hspace{23em}+\frac{1}{2} \bm{\xi}^{\mathsf{T}}\nabla^{2}_{\bm{\phi}}\mathbb{E}_{(\bm{s}, \bm{a})\sim \pi(\cdot)\rho_{\pi}^{\phi_{0}}}\left[\mathcal{W}_{(\bm{s},\bm{a})}(\bm{\phi}_{0})\right]\bm{\xi}\bm{I}\Bigg] \\ \nonumber
   &\hspace{0em} - \frac{1}{\sigma^{2}}\mathbb{E}_{\bm{\xi}\sim\mathcal{N}(\bm{0}, \sigma^{2}\bm{I})} \Bigg[\mathbb{E}_{(\bm{s}, \bm{a})\sim \pi(\cdot)\rho_{\pi}^{\phi_{0}}}\left[\mathcal{W}_{(\bm{s},\bm{a})}(\bm{\phi}_{0})\right]\bm{I} + \bm{\xi}^{\mathsf{T}}\nabla_{\bm{\phi}}\mathbb{E}_{(\bm{s}, \bm{a})\sim \pi(\cdot)\rho_{\pi}^{\phi_{0}}}\left[\mathcal{W}_{(\bm{s},\bm{a})}(\bm{\phi}_{0})\right]\bm{I} \\\nonumber
   & \hspace{23em}+\frac{1}{2}\bm{\xi}^{\mathsf{T}}\nabla^{2}_{\bm{\phi}}\mathbb{E}_{(\bm{s}, \bm{a})\sim \pi(\cdot)\rho_{\pi}^{\phi_{0}}}\left[\mathcal{W}_{(\bm{s},\bm{a})}(\bm{\phi}_{0})\right]\bm{\xi}\bm{I}\Bigg].
\end{align}
Now, we analyse each of the above terms separately. For ease of notation, we define the following variables: 
\begin{align*}
    \bm{g} &= \nabla_{\bm{\phi}}\mathbb{E}_{(\bm{s}, \bm{a})\sim \pi(\cdot)\rho_{\pi}^{\phi_{0}}}\left[\mathcal{W}_{(\bm{s},\bm{a})}(\bm{\phi}_{0})\right]     \hspace{10em}\bm{H} = \nabla^{2}_{\bm{\phi}}\mathbb{E}_{(\bm{s}, \bm{a})\sim \pi(\cdot)\rho_{\pi}^{\phi_{0}}}\left[\mathcal{W}_{(\bm{s},\bm{a})}(\bm{\phi}_{0})\right] \\
    \bm{A} & = \bm{\xi}\bm{\xi}^{\mathsf{T}}\bm{g}\bm{\xi}^{\mathsf{T}}     \hspace{20em}\bm{B} = \bm{\xi}\bm{\xi}^{\mathsf{T}}\bm{H}\bm{\xi}\bm{\xi}^{\mathsf{T}} \\
    c &= \bm{\xi}^{\mathsf{T}}\bm{H}\bm{\xi}.
\end{align*}
Starting with $\bm{A}$, we can easily see that any $(i,j)$ component can be written as $\bm{A} = \sum_{n=1}^{d_{2}} \bm{\xi}_{i}\bm{\xi}_{j}\bm{\xi}_{n}\bm{g}_{n}$. Therefore, the expectation under $\bm{\xi} \sim \mathcal{N}(\bm{0}, \sigma^{2}\bm{I})$ can be derived as: 
\begin{align*}
    \mathbb{E}_{\bm{\xi} \sim \mathcal{N}(\bm{0}, \sigma^{2}\bm{I})}\left[\bm{\xi}_{i}\bm{\xi}_{j}\bm{\xi}_{n}\bm{g}_{n}\right] = \bm{g}_{n}\mathbb{E}_{\bm{\xi}_{i}\sim \mathcal{N}(0, \sigma^{2})}\left[\bm{\xi}_{i}^{3}\right] = 0 \ \ \ \text{if $i = j = n$ and 0 otherwise.}
\end{align*}
Thus, we conclude that $\mathbb{E}_{\bm{\xi}\sim\mathcal{N}(\bm{0}, \sigma^{2}\bm{I})}[\bm{A}] = \bm{0}_{d_{2} \times d_{2}}$.

Continuing with the second term, i.e., $\bm{B}$, we realise that any $(i,j)$ component can be written as $\bm{B}_{i,j} = \sum_{n=1}^{d_{2}}\sum_{m=1}^{d_{2}}\bm{\xi}_{i}\bm{\xi}_{j}\bm{\xi}_{n}\bm{\xi}_{m}\bm{H}_{m,n}$. Now, we consider two cases: 
\begin{itemize}
    \item \underline{Diagonal Elements (i.e., when $i=j$):} The expectation under $\bm{\xi} \sim \mathcal{N}(\bm{0}, \sigma^{2}\bm{I})$  can be further split in three sub-cases
    \begin{itemize}
        \item \underline{Sub-Case I when $i=j=m=n$}: We have $\mathbb{E}_{\bm{\xi}\sim\mathcal{N}(0, \sigma^{2}\bm{I})}\left[\bm{\xi}_{i}\bm{\xi}_{j}\bm{\xi}_{n}\bm{\xi}_{m}\bm{H}_{m,n}\right] = \bm{H}_{i,i}\mathbb{E}_{\bm{\xi}_{i}\sim\mathcal{N}(0, \sigma^{2})} = 3\sigma^{4}\bm{H}_{i,i}.$
        \item \underline{Sub-Case II when $i=j \neq m = n$:} We have $\mathbb{E}_{\bm{\xi}\sim\mathcal{N}(0, \sigma^{2}\bm{I})}\left[\bm{\xi}_{i}\bm{\xi}_{j}\bm{\xi}_{n}\bm{\xi}_{m}\bm{H}_{m,n}\right] = \bm{H}_{m,m}\mathbb{E}_{\bm{\xi} \sim \mathcal{N}(\bm{0}, \sigma^{2}\bm{I})}\left[\bm{\xi}_{i}^{2}\bm{\xi}_{m}^{2}\right] = \sigma^{4}\bm{H}_{m,m}.$
        \item \underline{Sub-Case III when indices are all distinct:} We have $\mathbb{E}_{\bm{\xi}\sim\mathcal{N}(0, \sigma^{2}\bm{I})}\left[\bm{\xi}_{i}\bm{\xi}_{j}\bm{\xi}_{n}\bm{\xi}_{m}\bm{H}_{m,n}\right] = {0}.$
    \end{itemize}
    \begin{tcolorbox}[enhanced, arc=0pt,outer arc=0pt, colback=white, colframe=black, drop shadow={black,opacity=1}]
\textbf{\underline{Diagonal Elements Conclusion:}}
Using the above results we conclude that $\mathbb{E}_{\bm{\xi} \sim \mathcal{N}(\bm{0}, \sigma^{2}\bm{I})}[\bm{B}_{i,i}] = 2 \sigma^{4}\bm{H}_{i,i} + \sigma^{4}\text{trace}(\bm{H})$. 
\end{tcolorbox}
    \item \underline{Off-Diagonal Elements (i.e., when $i\neq j$):} The above analysis is now repeated for computing the expectation of the off-diagonal elements of matrix $\bm{B}$. Similarly, this can also be split into three sub-cases depending on indices: 
    \begin{itemize}
        \item \underline{Sub-Case I when $i = m \neq j = n$:} We have $\mathbb{E}_{\bm{\xi}\sim\mathcal{N}(\bm{0}, \sigma^{2}\bm{I})}\left[\bm{\xi}_{i}\bm{\xi}_{j}\bm{\xi}_{n}\bm{\xi}_{m}\bm{H}_{m,n}\right] = \bm{H}_{i,j}\mathbb{E}_{\bm{\xi}\sim \mathcal{N}(\bm{0}, \sigma^{2}\bm{I})}\left[\bm{\xi}_{i}^{2}\bm{\xi}_{j}^{2}\right] = \sigma^{4}\bm{H}_{i,j}.$
        \item \underline{Sub-Case II when $i = n \neq j = m$:} We have $\mathbb{E}_{\bm{\xi}\sim\mathcal{N}(\bm{0}, \sigma^{2}\bm{I})}\left[\bm{\xi}_{i}\bm{\xi}_{j}\bm{\xi}_{n}\bm{\xi}_{m}\bm{H}_{m,n}\right] = \bm{H}_{j,i}\mathbb{E}_{\bm{\xi}\sim \mathcal{N}(\bm{0}, \sigma^{2}\bm{I})}\left[\bm{\xi}_{i}^{2}\bm{\xi}_{j}^{2}\right] = \sigma^{4}\bm{H}_{j, i}.$
        \item \underline{Sub-Case III when indices are all distinct:} We have $\mathbb{E}_{\bm{\xi}\sim\mathcal{N}(\bm{0}, \sigma^{2}\bm{I})}\left[\bm{\xi}_{i}\bm{\xi}_{j}\bm{\xi}_{n}\bm{\xi}_{m}\bm{H}_{m,n}\right] = 0.$
    \end{itemize}
\begin{tcolorbox}[enhanced, arc=0pt,outer arc=0pt, colback=white, colframe=black, drop shadow={black,opacity=1}]
    \textbf{\underline{Off-Diagonal Elements Conclusion:}}
Using the above results and due to the symmetric properties of $\bm{H}$, we conclude that $\mathbb{E}_{\bm{\xi} \sim \mathcal{N}(\bm{0}, \sigma^{2}\bm{I)}}\left[\bm{B}_{i,j}\right] = 2\sigma^{4}\bm{H}_{i,j}$
\end{tcolorbox}
\end{itemize}
Finally, analysing $c$, one can realise that $\mathbb{E}_{\bm{\xi}\sim \mathcal{N}(\bm{0}, \sigma^{2}\bm{I})}[c] = \mathbb{E}_{\bm{\xi}\sim \mathcal{N}(\bm{0}, \sigma^{2}\bm{I})}\left[\sum_{i=1}^{d_{2}}\sum_{j=1}^{d_{2}}\bm{\xi}_{i}\bm{\xi}_{j}\bm{H}_{i,j}\right]=\sigma^{2}\text{trace}(\bm{H})$. 

Substituting the above conclusions back in the original approximation in Equation~\ref{Eq:HessianApprox}, and using the linearity of the expectation we can easily achieve the statement of the proposition.
\end{proof}

With the above two propositions, we can now perform the updates in Algorithm~\ref{Algo:Main} without the need for performing explicit model learning. This is true as Propositions 6 and 7 devise procedure where gradient and Hessian estimates can be simply based on simulator value evaluations while perturbing $\bm{\phi}$ and $\bm{\phi}_{0}$. It is important to note that in order to apply the above, we are required to be able to evaluate $\mathbb{E}_{(\bm{s},\bm{a})\sim \pi(\cdot)\rho_{\pi}^{\bm{\phi}_{0}}(\cdot)}\left[\mathcal{W}_{(\bm{s}, \bm{a})}(\bm{\phi}_{0})\right]$ under random $\bm{\xi}$ perturbations sampled from $\mathcal{N}(\bm{0}, \sigma^{2}\bm{I})$. An empirical estimate of the $p$-Wasserstein distance between two measures $\mu$ and $\nu$ can be performed by computing the $p$-Wasserstein distance between the empirical distributions evaluated at sampled data. That is, one can approximation $\mu$ by $\mu_{n} = \frac{1}{n} \sum_{i=1}^{n}\delta_{\bm{x}_{i}}$ where $\bm{x}_{i}$ are identically and independently distributed according to $\mu$. Approximating $\nu_{n}$ similarly, we then realise that\footnote{In case the dynamics are assumed to be Gaussian, a similar procedure can be followed or a closed form can be used, see \cite{Takatsu_wassersteingeometry}.} $\mathcal{W}_{2}(\mu, \nu) \approx \mathcal{W}_{2}(\mu_{n}, \nu_{n})$.

\section{Related work}\label{Sec:Related_work}
In this section we review some of the related literature. There is a common theme running through previous works: the formulation of the problem as a max-min (or min-max, depending on whether the goal is to maximise for rewards or minimise for costs) problem. This game-theoretic view is natural formulation when viewing nature as an adversary. Thus, it will be no surprise that the papers discussed below mainly take this approach or close variations of it.

There is a long-standing thread of research on robustness in the classical control community, and the literature in this area is vast, with the $\mathcal{H}_\infty$ method being a standard approach \cite[]{doyle2013feedback}. This approach was introduced into reinforcement learning by \cite{Morimoto:2005:RRL:1119345.1119349}. In that paper, a continuous time reinforcement learning setting was studied for which a max-min problem was formulated involving a modified value function, the optimal solutions of which can be determined by solving Hamilton-Jacobi-Isaacs (HJI) equation. 

There is also a line of work on robust MDPs, amongst which are \cite{iyengar2005robust, nilim2005robust, wiesemann2013robust, tirinzoni2018policy, petrik2019beyond}. In particular, \cite{yang2017convex} uses the Wasserstein distance to define uncertainty sets of dynamics in similar way to this work, that is, in an $\epsilon$-ball around a particular dynamics (referred to as \emph{nominal} distribution in that paper). The paper shows that an optimal Markov control policy is possible the max-min Bellman equations and shows how convex-optimisation techniques can be applied to solve it.  

Whilst valuable in their own right, these approaches are not sufficient for the RL setting due to the need in the latter case to give efficient solutions for large state and action spaces, and the fact that the dynamics are not known \emph{a priori}. We emphasise once again that in our setting, cannot explicitly define the MDP of the reference dynamics, since we do not assume knowledge of it. We assume only that we can sample from it, which is the standard assumption made in RL.

Reasonably, one might expect that model-based reinforcement learning may be a plausible route to address robustness. In \cite{pmlr-v80-asadi18a}, the learning of Lipschitz continuous models is addressed, and a bound on multi-step prediction error is given in terms of the Wasserstein distance. However, the major stumbling-block with model-based RL techniques is that in high-dimensional state  building models that are sufficient for controlling an agent can suffer greatly from model mis-specification or excessive computational costs (e.g., as in training a Gaussian Process, see, e.g., \cite{rasmussen2003gaussian, deisenroth2011pilco}). 

We now discuss some papers closer in objective and/or technique to our own. 
\cite{rajeswaran2016epopt} approaches the robustness problem by training on an ensemble of dynamics in order to be deployed on a target environment. The algorithm introduced, Ensemble Policy Optimisation (EPOpt), alternates between two phases: (i) given a distribution over dynamics for which simulators (or models) are available (the source domain), train a policy that performs well for the whole distribution; (ii) gather data from the deployed environment (target domain) to adapt the distribution. The objective is not max-min, but a softer variation defined by conditional value-at-risk (CVaR). The algorithm samples a set of dynamics $\{\bm{\phi}_k\}$ from a distribution over dynamics $\mathcal{P}_{\bm{\psi}}$, and for each dynamics $\bm{\phi}_k$, it samples a trajectory using the current policy parameter $\bm{\theta}_i$. It then selects the worst performing $\epsilon$-fraction of the trajectories to use to update the policy parameter. Clearly this process bears some resemblance to our algorithm, but there is a crucial difference: our algorithm takes \emph{descent steps} in the $\bm{\phi}$ space. The difference if important when the dynamics parameters sit in a high-dimensional space, since in that case, optimisation-from-sampling could demand a considerable number of samples. A counter argument against our technique might be that our zero'th-order method for estimating gradients and Hessians also requires sampling in high dimensions. This is, indeed the case, but obtaining localised estimates (as gradients and Hessians are local properties) could be easier than global properties (the worse set of parameters in the high-dimensional space). In any case, our experiments demonstrate our algorithm performs well even in these high dimensions. The experiments of \cite{rajeswaran2016epopt} are on Hopper and HalfCheetah, in which their algorithm is compared to TRPO.  We note that we were were unable to find the code for this paper, and did not attempt to implement it ourselves.

The CVaR criterion is also adopted in \cite{pmlr-v70-pinto17a}, in which, rather than sampling trajectories and finding a quantile in terms of performance, two policies are trained simultaneously: a ``protagonist'' which aims to optimise performance, and an adversary which aims to disrupt the protagonist. The protagonist and adversary train alternatively, with one being fixed whilst the other adapts. The action space for the adversary, in the tests documented in the paper includes forces on the entities (InvertedPendulum, HalfCheetah, Swimmer, Hopper, Walker2D) that aim to destabalise it. We made comparisons against this algorithm in our experiments.

More recently, \cite{pmlr-v97-tessler19a} studies robustness with respect to action perturbations. There are two forms of perturbation addressed: (i) Probabilistic Action Robust MDP (PR-MDP), and (ii) Noisy Action Robust MDP (NR-MDP). In PR-MDP, when an action is taken by an agent, with probability $\alpha$, a different, possibly adversarial action is taken instead. In NR-MDP, when an action is taken, a perturbation is added to the action itself. Like \cite{rajeswaran2016epopt} and \cite{pmlr-v70-pinto17a}, the algorithm is suitable for applying deep neural networks, and the paper reports experiments on InvertedPendulum, Hopper, Walker2d and Humanoid. We tested against PR-MDP in some of our experiments, and found it to be lacking in robustness (see Section \ref{Sec:Exps}, Figure \ref{fig:myfig} and Figure \ref{fig:my_label}). 

In \cite{lecarpentier2019non} a non-stationary Markov Decision Process model is considered, where the dynamics can change from one time step to another. The constraint is based on Wasserstein distance, specifically, the Wasserstein distance between dynamics at time $t$ and $t'$ is bounded by $L|t-t'|$, i.e., is $L$-Lipschitz with respect to time, for some constant $L$. They approach the problem by treating nature as an adversary and implement a Minimax algorithm. The basis of their algorithm is that due to the fact that the dynamics changes slowly (due to the Lipschitz constraint), a planning algorithm can project into the future the scope of possible future dynamics and plan for the worst. The resulting algorithm, known as \emph{Risk Averse Tree Search}, is - as the name implies - a tree search algorithm. It operates on a sequence ``snapshots'' of the evolving MDP, which are instances of the MDP at points in time. The algorithm is tested on small grid world, and does not appear to be readily extendible to the continuous state and action scenarios our algorithm addresses. 

To summarise, our paper uses the Wasserstein distance for quantifying variations in possible dynamics, in common with \cite{lecarpentier2019non}, but is suited to applying deep neural networks for continuous state and action spaces. Our algorithm does not require a full dynamics available to it, merely a parameterisable dynamics. It competes well with the above papers, and operates well for high dimensional problems, as evidenced by the experiments.

\section{Experiments \& Results}\label{Sec:Exps}
We evaluate $\text{W}\text{R}^{2}\text{L}$ on a variety of continuous control benchmarks from the MuJoCo environment. Dynamics in our benchmarks were parameterised by variables defining physical behaviour, e.g., density of the robot's torso, friction of the ground, and so on. We consider both low and high dimensional dynamics and demonstrate that our algorithm outperforms state-of-the-art from both standard and robust reinforcement learning. We are chiefly interested in policy generalisation across environments with varying dynamics, which we measure using average test returns on novel systems. The comparison against standard reinforcement learning algorithms allows us to understand whether lack of robustness is a critical challenge for sequential decision making, while comparisons against robust algorithms test if we outperform state-of-the-art that considered a similar setting to ours. From standard algorithms, we compare against proximal policy optimisation (PPO)~\citep{schulman2017proximal}, and trust region policy optimisation (TRPO)~\citep{DBLP:journals/corr/SchulmanLMJA15}; an algorithm based on natural actor-crtic~\citep{Peters:2008:NA:1352927.1352986, DBLP:journals/corr/abs-1902-02823}. From robust algorithms, we demonstrate how $\text{W}\text{R}^{2}\text{L}$ favours against robust adversarial reinforcement learning (RARL)~\citep{pmlr-v70-pinto17a}, and action-perturbed Markov decision processes (PR-MDP) proposed in~\citep{pmlr-v97-tessler19a}. 

It is worth noting that we attempted to include deep deterministic policy gradients (DDPG)~\citep{Silver:2014:DPG:3044805.3044850} in our comparisons. Results including DDPG were, however, omitted as it failed to show any significant robustness performance even on relatively simple systems, such as the inverted pendulum; see results reported in Appendix~\ref{App:InvBar}. During initial trials, we also performed experiments parameterising models using deep neural networks. Results demonstrated that these models, though minimising training data error, fail to provide valid physics-grounded dynamics. For instance, we arrived at inverted pendula models that vary pole angles without exerting any angular speeds. This problem became even more apparent in high-dimensional systems, e.g., Hopper, Walker, etc due to the increased number of possible minima. As such, results presented in this section make use of our zero-order method that can be regarded as a scalable alternative for robust solutions.

\subsection{MuJoCo benchmarks} \label{sec:mujoco}
Contrary to other methods rooted in model-based reinforcement learning, we evaluate our method both in low and \emph{high-dimensional} MuJuCo tasks \citep{Mujuco}. We consider a variety of systems including CartPole, Hopper, and Walker2D; all of which require direct joint-torque control. Keeping with the generality of our method, we utilise these dynamical as-is with no additional alterations. Namely, we use the exact setting of these benchmarks as that shipped with OpenAI gym without any reward shaping, state-space augmentation, feature extraction, or any other modifications of-that-sort. For clarity, we summarise variables parameterising dynamics in Table \ref{table:dynamic_param}, and detail specifics next. 

\paragraph{CartPole:} The goal of this classic control benchmark is to balance a pole by driving a cart along a rail. The state space is composed of the position $x$ and velocity $\Dot{x}$ of the cart, as well as the angle $\theta$ and angular velocities of the pole $\Dot{\theta}$. We consider two termination conditions in our experiments: 1) pole deviates from the upright position beyond a pre-specified threshold, or 2) cart deviates from its zeroth initial position beyond a certain threshold. To conduct robustness experiments, we parameterise the dynamics of the CartPole by the pole length $l_p$, and test by varying $l_{p} \in [0.3, 3]$. 

\textbf{Hopper:} In this benchmark, the agent is required to control a hopper robot to move forward without falling. The state of the hopper is represented by positions, $\{x, y, z\}$, and linear velocities, $\{\Dot{x}, \Dot{y}, \Dot{z}\}$, of the torso in global coordinate, as well as angles, $\{\theta_i\}_{i=0}^2$, and angular speeds, $\{\Dot{\theta}_i\}_{i=0}^2$, of the three joints. During training, we exploit an early-stopping scheme if ``unhealthy'' states of the robot were visited. Parameters characterising dynamics included densities $\{\rho_i\}_{i=0}^3$ of the four links, armature $\{a_i\}_{i=0}^2$ and damping $\{\zeta_i\}_{i=0}^2$ of three joints, and the friction coefficient $\mu_g$. To test for robustness, we varied both frictions and torso densities leading to significant variations in dynamics. We further conducted additional experiments while varying all 11 dimensional specification parameters.  

\textbf{Walker2D:} This benchmark is similar to Hopper except that the controlled system is a biped robot with seven bodies and six joints. Dimensions for its dynamics are extended accordingly as reported in Table \ref{table:dynamic_param}. Here, we again varied the torso density for performing robustness experiments in the range $\rho_{0} \in [500, 3000]$.  

\textbf{Halfcheetah:} This benchmark is similar to the above except that the controlled system is a two-dimensional slice of a three-dimensional cheetah robot. Parameters specifying the simulator consist of 21 dimensions, with 7 representing densities. In our two-dimensional experiments we varied the torso-density and floor friction, while in high-dimensional ones, we allowed the algorithm to control all 21 variables.

\subsection{Experimental protocol}
\begin{table}
\renewcommand{\arraystretch}{1.2}
\begin{tabular}{c|ccc} 
     &  1D experiment & 2D experiment & High-dimensional experiment \\ \hline
     Inverted Pendulum & $l_p$ & None & None\\
     Hopper & $\rho_0$ & $\{\rho_0, \mu_g\}$ & $\{\rho_i\}_{i=0}^3 \cup \{a_i\}_{i=0}^2 \cup \{\zeta_i\}_{i=0}^2 \cup \mu_g$ \\
     Walker2D & $\rho_0$ & $\{\rho_0, \mu_g\}$ & $\{\rho_i\}_{i=0}^6 \cup \{a_i\}_{i=0}^5 \cup \{\zeta_i\}_{i=0}^5 \cup \mu_g$ \\
    HalfCheetah & None & $\{\rho_0, \mu_g\}$ & $\{\rho_i\}_{i=0}^7 \cup \{a_i\}_{i=0}^5 \cup \{\zeta_i\}_{i=0}^5 \cup \mu_g$ \\
     [0.3cm]
\end{tabular}
\caption{Parameterisation of dynamics. See section \ref{sec:mujoco} for the physical meaning of these parameters.}
\label{table:dynamic_param}
\end{table}

Our experiments included training and a testing phases. During the training phase we applied Algorithm \ref{Algo:Main} for determining robust policies while updating transition model parameters according to the min-max formulation. Training was performed independently for each of the algorithms on the relevant benchmarks while ensuring best operating conditions using hyper-parameter values reported elsewhere \citep{schulman2017proximal, pmlr-v70-pinto17a, pmlr-v97-tessler19a}.   

For all benchmarks, policies were represented using parametrised Gaussian distributions with their means given by a neural network and standard derivations by a group of free parameters. The neural network consisted of two hidden layers with 64 units and hyperbolic tangent activations in each of the layers. The final layer exploited linear activation so as to output a real number. Following the actor-critic framework, we also trained a standalone critic network having the same structure as that of the policy. 

For each policy update, we rolled-out in the current worst-case dynamics to collect a number of transitions. The number associated to these transitions was application-dependent and varied between benchmarks in the range of 5,000 to 10,000. The policy was then optimised (i.e., Phase II of Algorithm \ref{Algo:Main}) using proximal policy optimization with a generalised advantage estimation. To solve the minimisation problem in the inner loop of Algorithm \ref{Algo:Main}, we sampled a number of dynamics from a diagonal Gaussian distribution that is centered at the current worst-case dynamics model. The number of sampled dynamics and the variance of the sampled distributions depended on both the benchmark itself, and well as the dimensions of the dynamics. Gradients needed for model updates were estimated using the results in Propositions 7 and 8. Finally, we terminated training when the policy entropy dropped below an application-dependent threshold.

When testing, we evaluated policies on unseen dynamics that exhibited simulator variations as described earlier. We measured performance using returns averaged over 20 episodes with a maximum length of 1,000 time steps on testing environments. We note that we used non-discounted mean episode rewards to compute such averages.

\subsection{Comparison Results \& Benchmarking}
This section summarises robustness results showing that our method significantly outperforms others from both standard and robust reinforcement learning in terms of average testing returns as dynamics vary. 

\paragraph{Results with One-Dimensional Model Variation:} Figure \ref{fig:myfig} shows the robustness of policies on a simple inverted pendulum while varying the pole length in the ranges from $0.3$ to $3.0$. For a fair comparison, we trained two standard policy gradient methods (TRPO \citep{DBLP:journals/corr/SchulmanLMJA15} and PPO \citep{schulman2017proximal}), and two robust RL algorithms (RARL \citep{pmlr-v70-pinto17a}, PR-MDP \citep{pmlr-v97-tessler19a}) with the reference dynamics preset by our algorithm. The range of evaluation parameters was intentionally designed to include dynamics out of the $\epsilon$-Wasserstein ball. Clearly, $\text{W}\text{R}^{2}\text{L}$ outperforms all baselines in this  benchmark. 
\begin{wrapfigure}{r}{0.58\linewidth}
\centering
 \includegraphics[scale=.58, trim = {0em 0em 0em 0em}, clip]{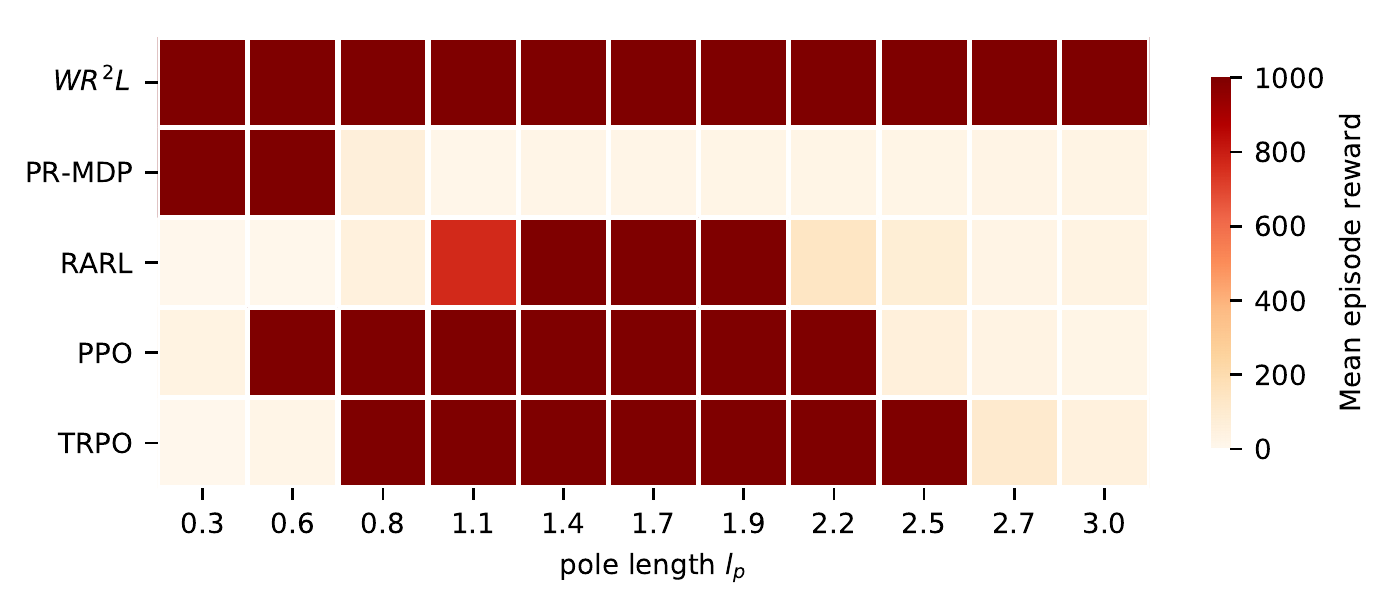}
\caption{Robustness results on the inverted pendulum demonstrating that our method outperforms state-of-the-art in terms of average test returns.}
\label{fig:myfig}
\end{wrapfigure}

Given successful behaviour in low-dimensional state representations, we performed additional experiments on the Hopper and Walker systems to assess robustness against model changes in high-dimensional environments. Figure \ref{fig:my_label} illustrates these results depicting that our method is again capable of outperforming others including RARL and PR-MDP. It is also interesting to realise that in high-dimensional environments, our algorithm exhibits a trade-off between robustness and optimality due to the min-max definition of $\text{W}\text{R}^{2}\text{L}$'s objective. 

\begin{tcolorbox}[enhanced, arc=0pt,outer arc=0pt, colback=white, colframe=black, drop shadow={black,opacity=1}]
\textbf{Experimental Conclusion I:} From the above, we conclude that $\text{W}\text{R}^{2}\text{L}$ outperforms others when one-dimensional simulator variations are considered. 
\end{tcolorbox}

\begin{figure}[h!]
    \centering
    \includegraphics[scale = .4, trim = {0em 27em 0em 27em}, clip]{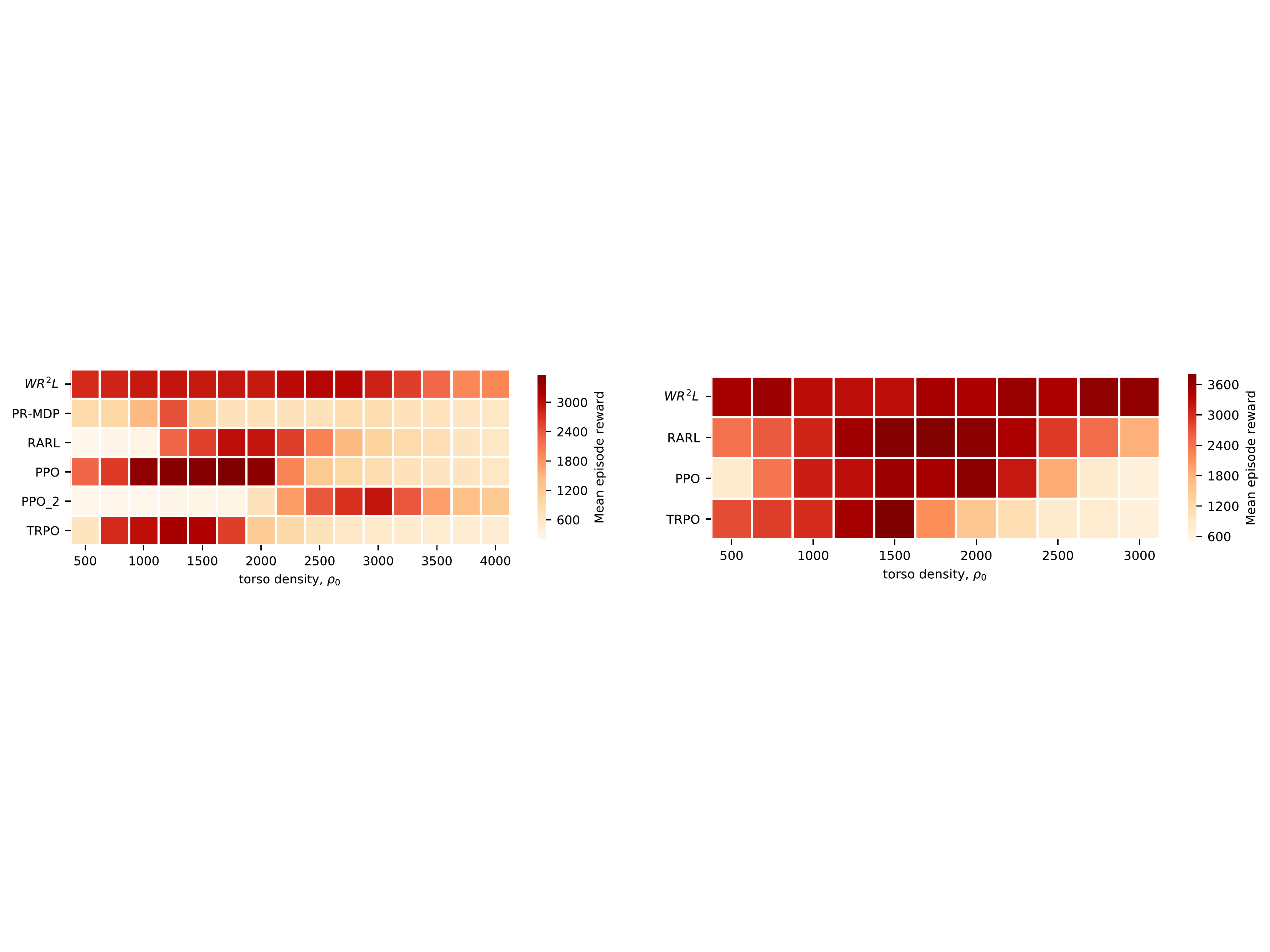}
    \caption{Robustness results on Hopper (left) and Walker (right) systems demonstrating that our method outperforms others significantly in terms of average test returns as torso densities vary. It is also interesting to realise that due to the robust problem formulation, our algorithm exhibits a trade-off between optimality and generalisation. Hopper results are with a reference $\rho_0=1750$; PPO$_2$ uses the same implementation as PPO but trained with $\rho_0=3000$. Walker results are attained with a reference model of $\rho_0=1750$. }
    \label{fig:my_label}
\end{figure}

\begin{figure}[h!]
     \centering
     \begin{subfigure}[b]{0.3\textwidth}
         \centering
         \includegraphics[width=\textwidth]{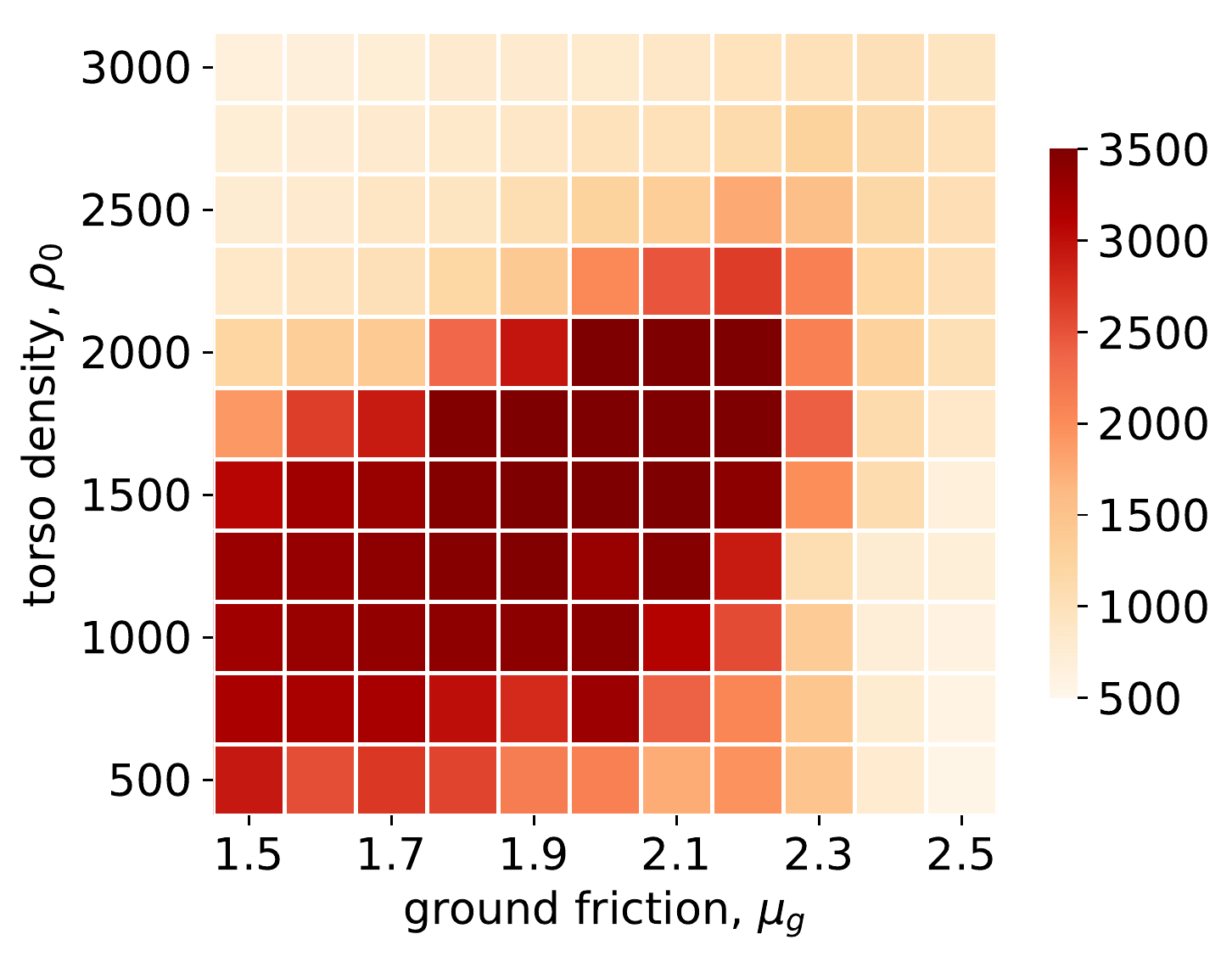}
         \caption{Hopper - $\epsilon=0$}
     \end{subfigure}
     \hfill
     \begin{subfigure}[b]{0.3\textwidth}
         \centering
         \includegraphics[width=\textwidth]{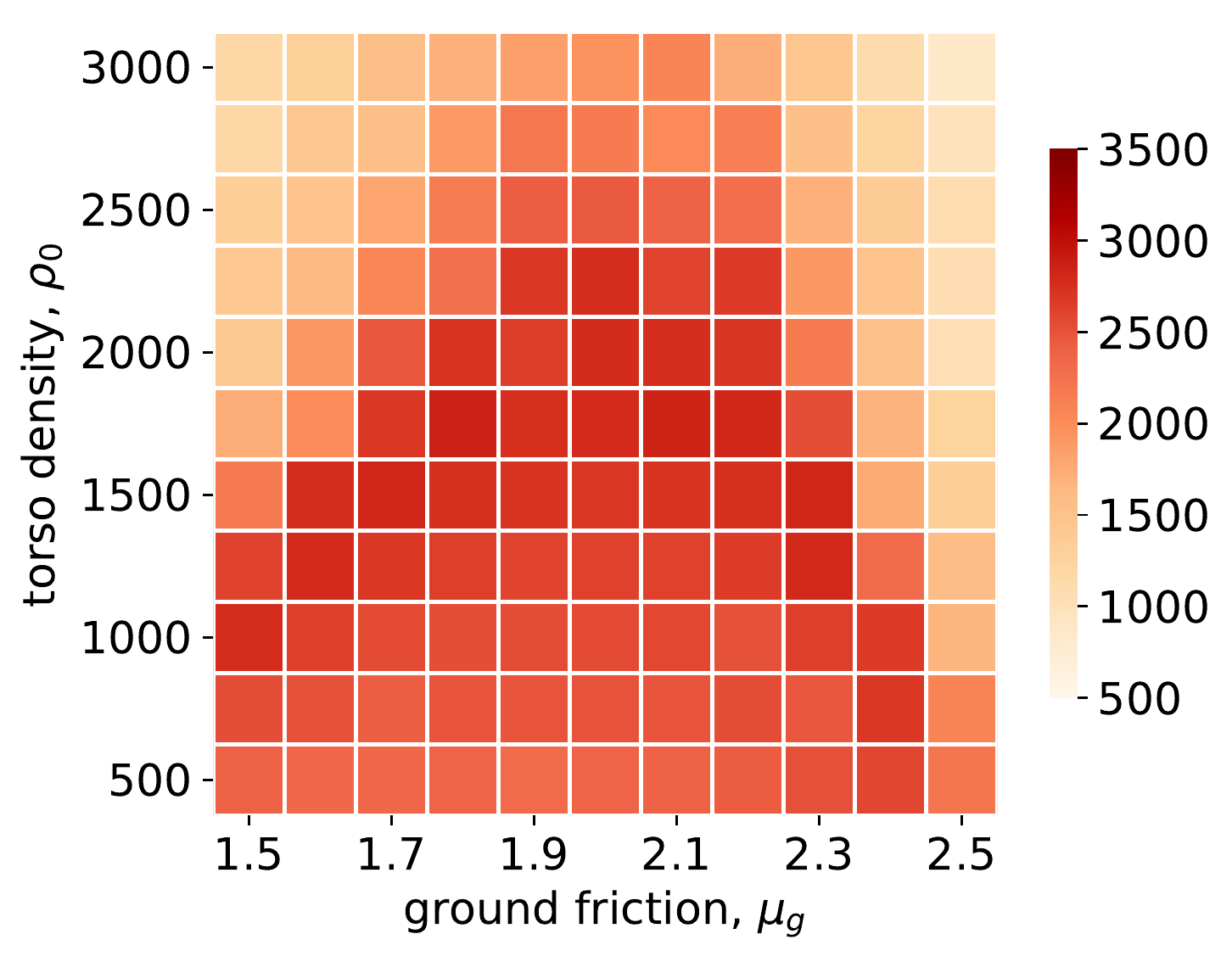}
         \caption{Hopper - $\epsilon=0.003$}
     \end{subfigure}
     \hfill
     \begin{subfigure}[b]{0.3\textwidth}
         \centering
         \includegraphics[width=\textwidth]{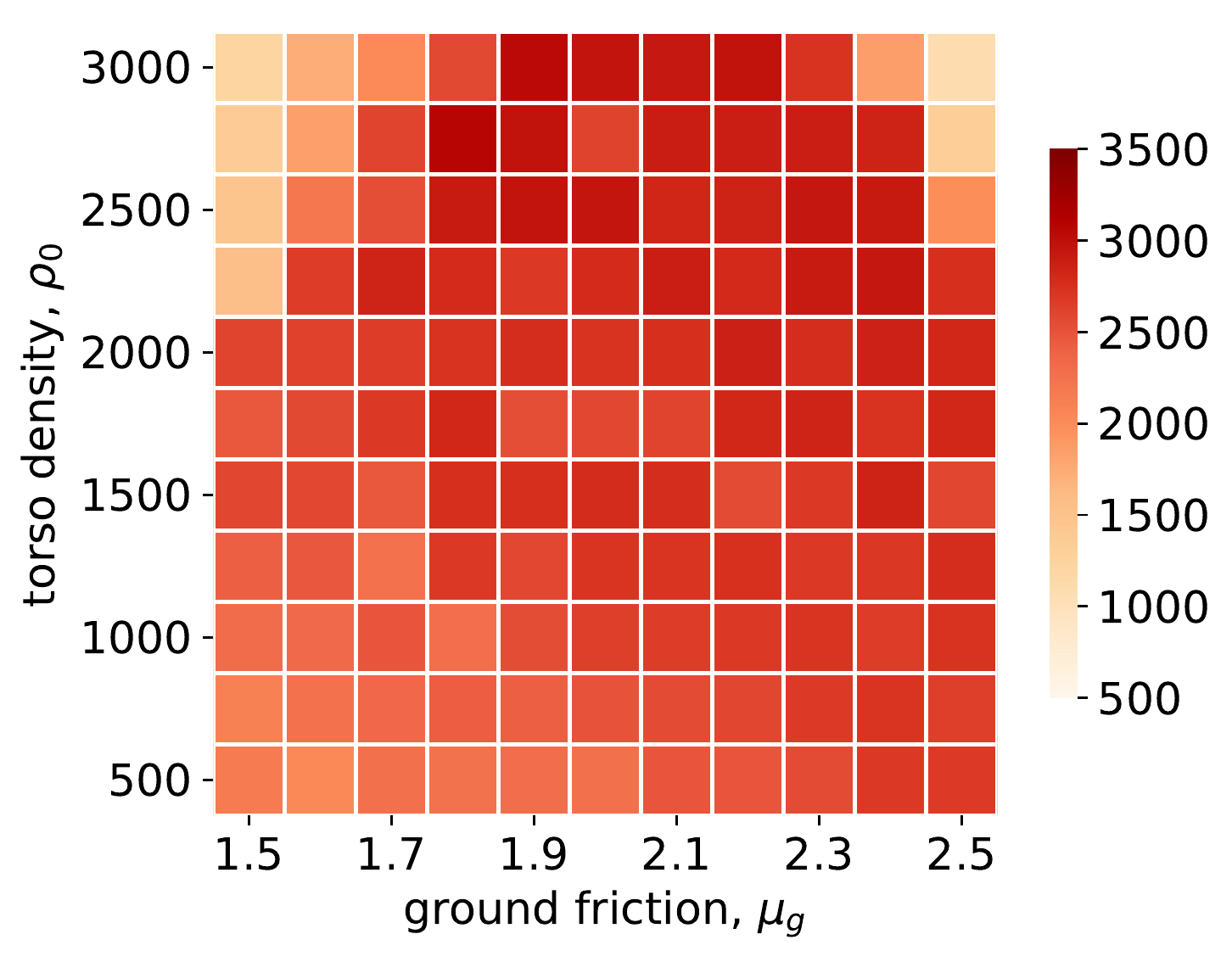}
         \caption{Hopper - $\epsilon=0.015$}
         \end{subfigure}
    \hfill
    \begin{subfigure}[b]{0.3\textwidth}
         \centering
         \includegraphics[width=\textwidth]{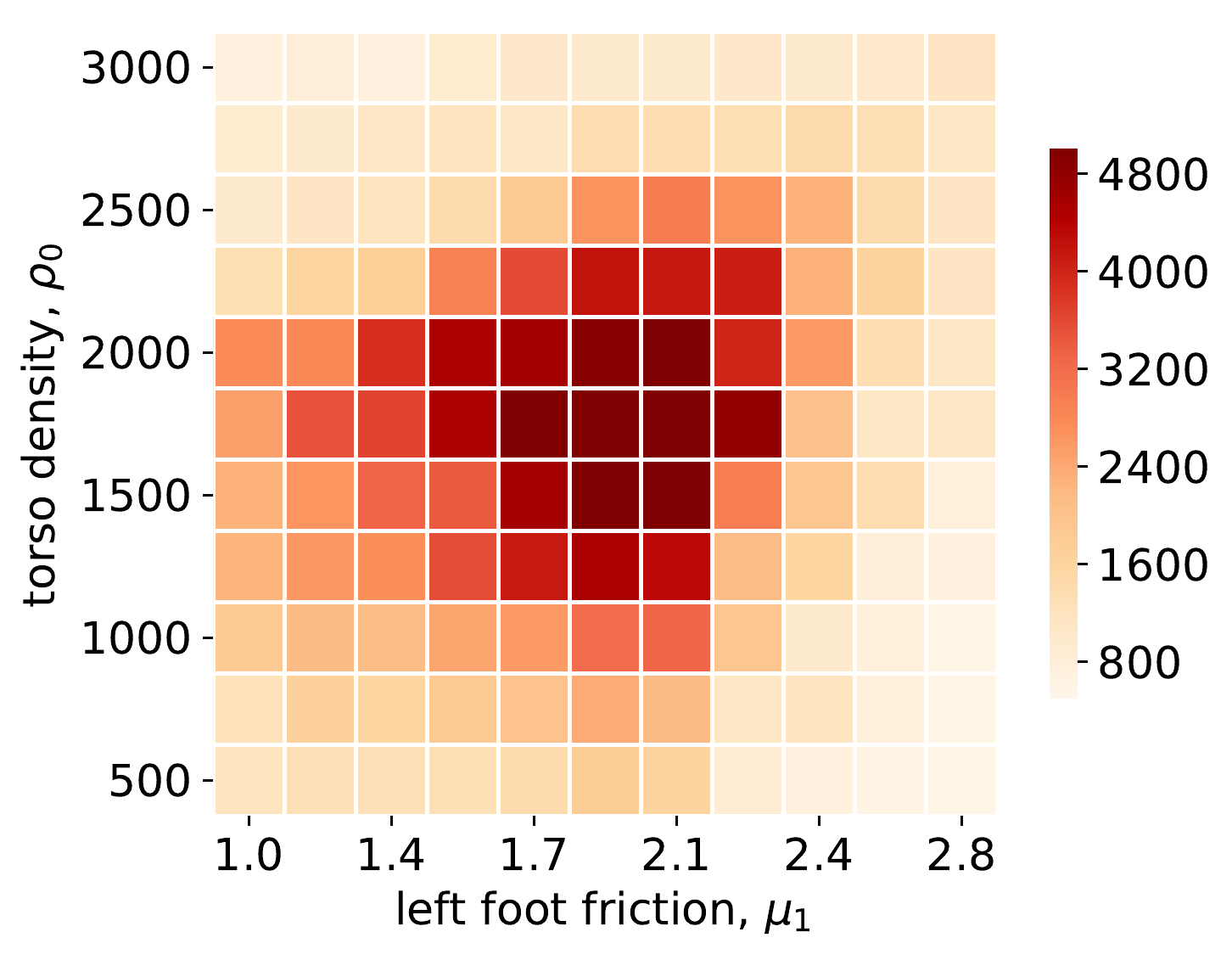}
         \caption{Walker - $\epsilon=0$}
         \end{subfigure}
         \hfill
    \begin{subfigure}[b]{0.3\textwidth}
         \centering
         \includegraphics[width=\textwidth]{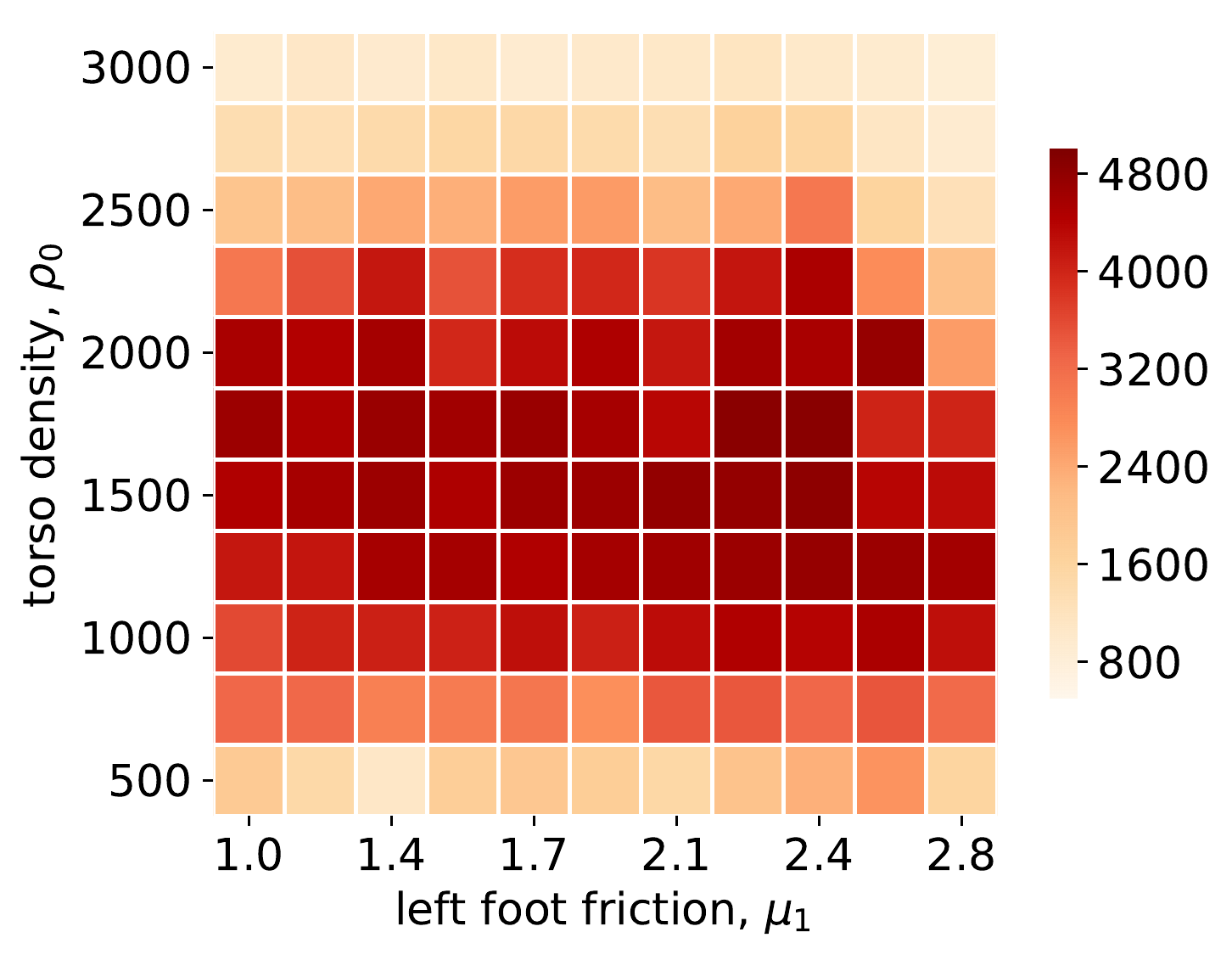}
         \caption{Walker - $\epsilon=0.1$}     
    \end{subfigure}
    \hfill     
    \begin{subfigure}[b]{0.3\textwidth}
         \centering
         \includegraphics[width=\textwidth]{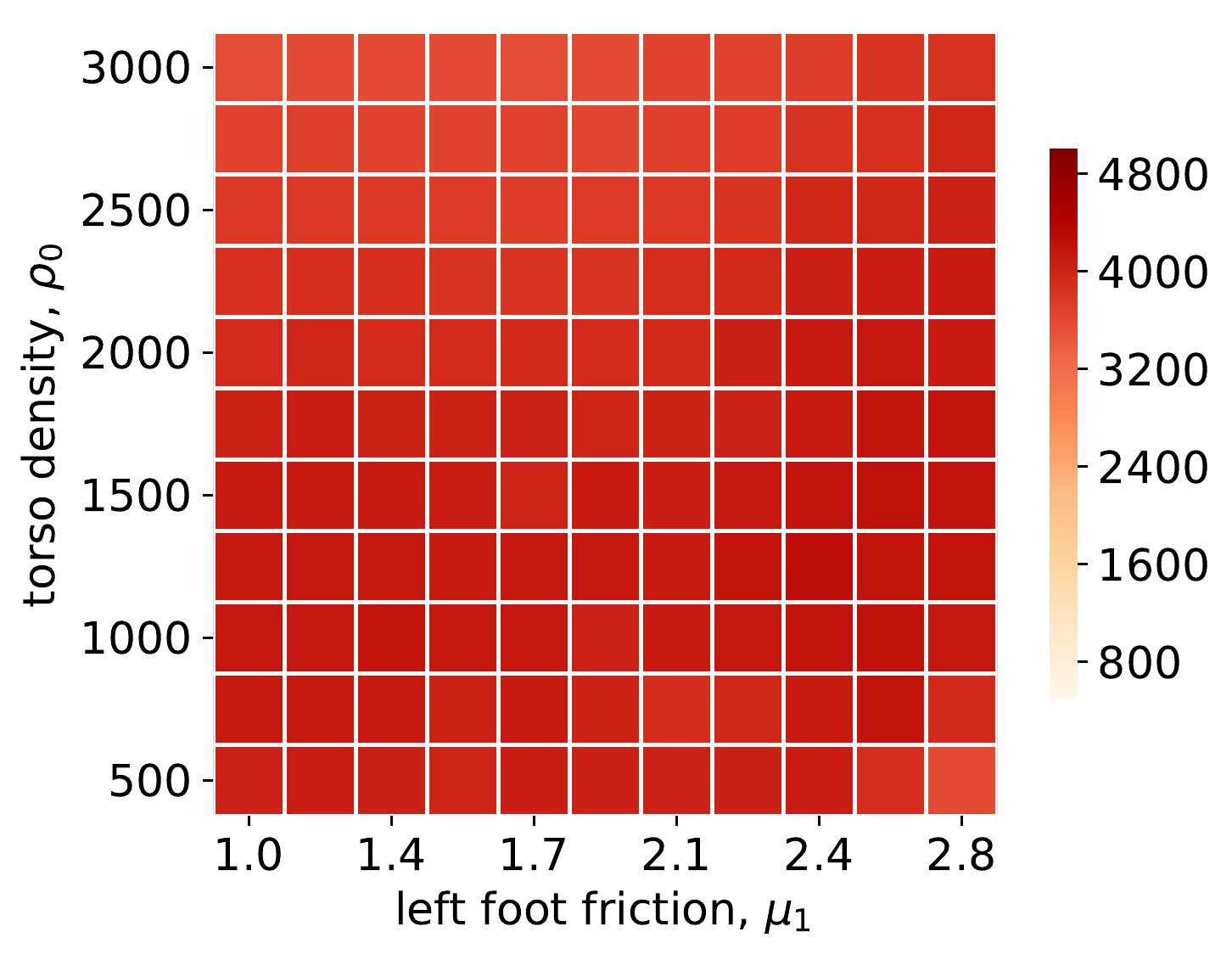}
         \caption{Walker - $\epsilon=1.0$}   
     \end{subfigure}         
     \hfill     
    \begin{subfigure}[b]{0.3\textwidth}
         \centering
         \includegraphics[width=\textwidth]{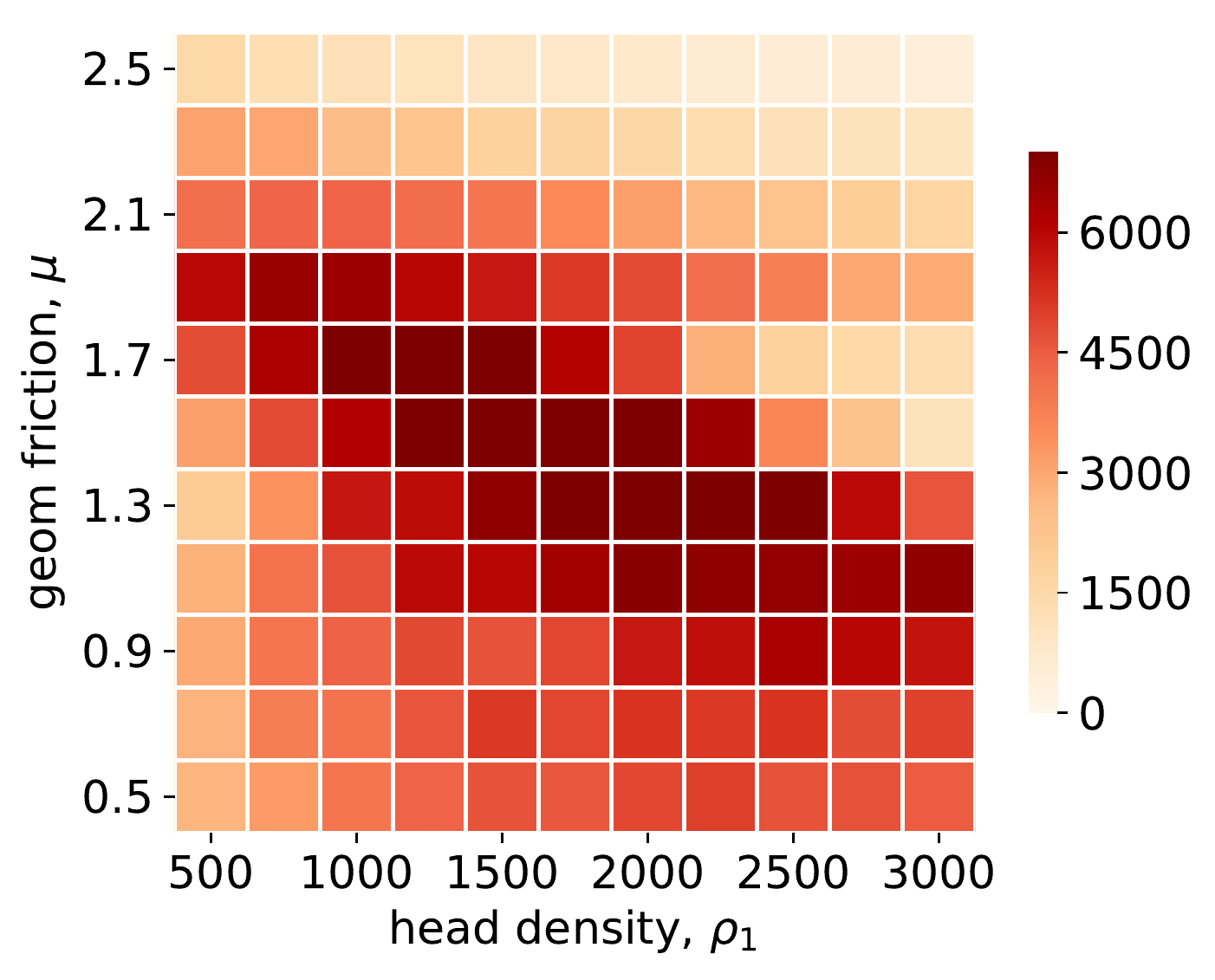}
         \caption{Cheetah - $\epsilon=0.0$}          
    \end{subfigure}
    \hfill \hfill \hfill    
    \begin{subfigure}[b]{0.3\textwidth}
         \centering
         \includegraphics[width=\textwidth]{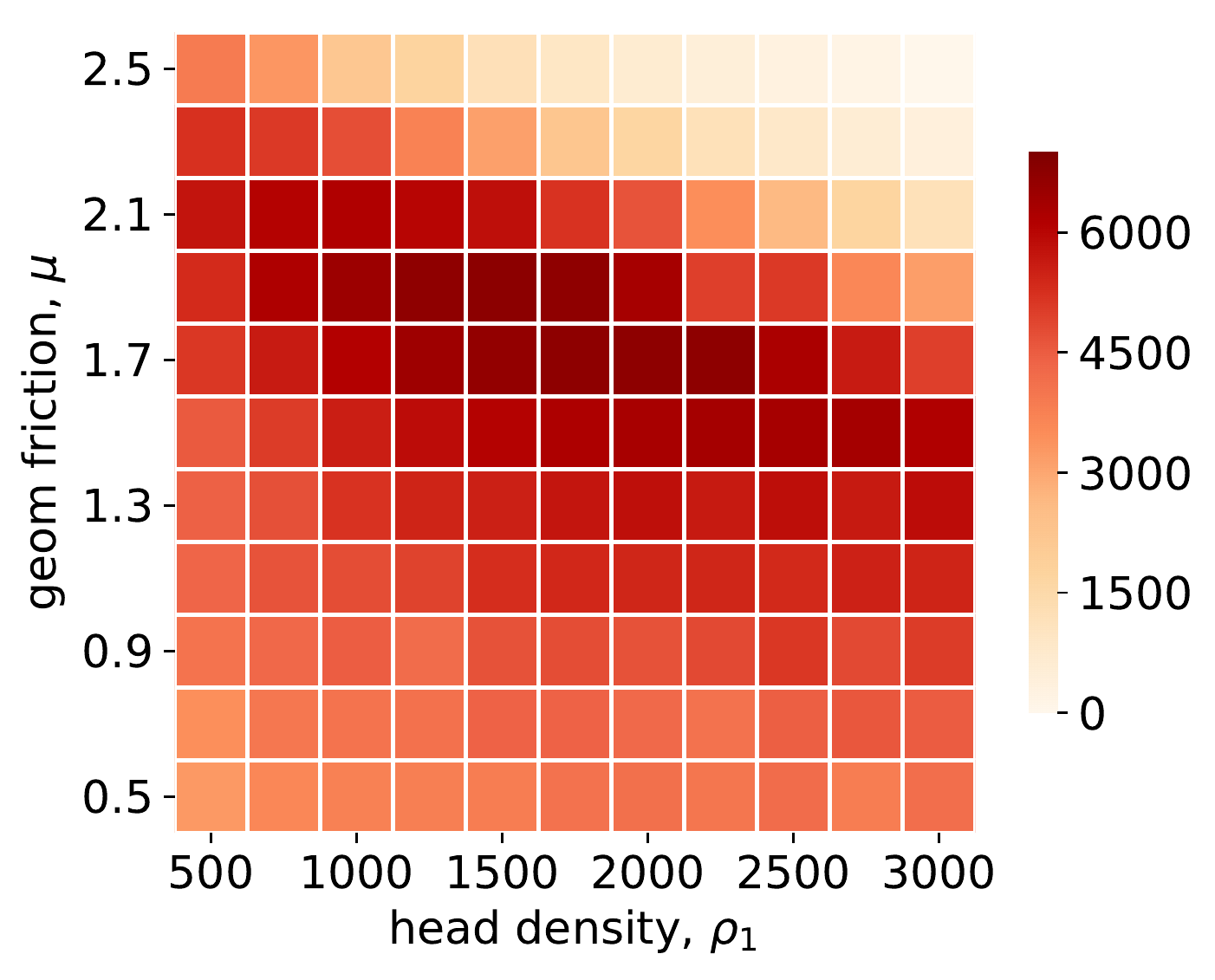}
    \caption{Cheetah - $\epsilon=0.005$}          
    \end{subfigure}
    \hfill \hfill 
    \begin{subfigure}[b]{0.3\textwidth}
         \centering
         \includegraphics[width=\textwidth]{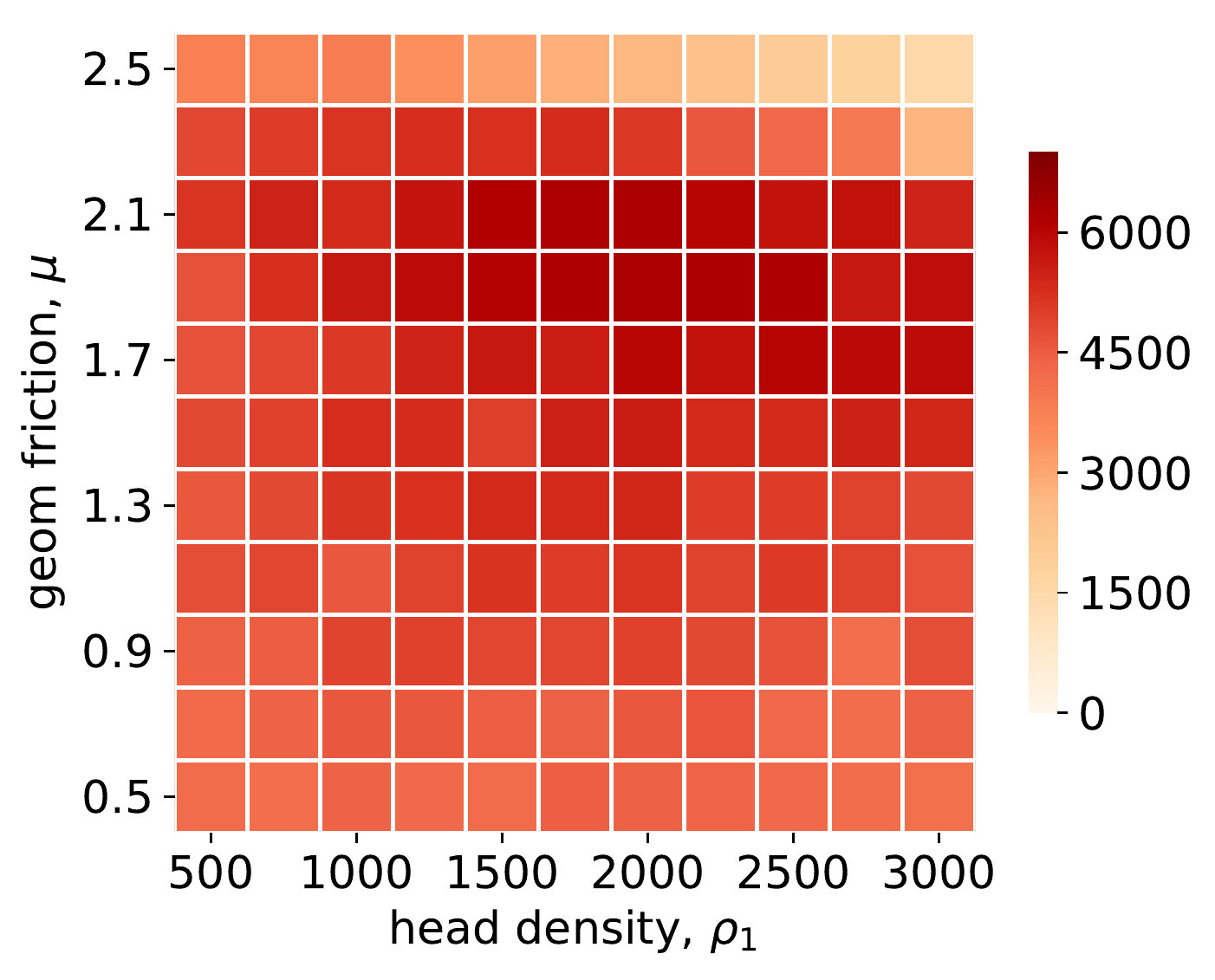}
    \caption{Cheetah - $\epsilon=0.03$}          
    \end{subfigure}
    
    \caption{Results on various benchmarks. Top row represents Hopper results, middle is concerned with Walker, and bottom denotes HalfCheetah. These graphs depict test returns as a function of changes in dynamical parameters and Wasserstein distance. These graphs again show that $\text{W}\text{R}^{2}\text{L}$ outperforms PPO (i.e., when $\epsilon =0$) and that its robustness improves as $\epsilon$ increases.}
    \label{fig:exp_2d}
\end{figure}

Furthermore, one would expect that such advantages increase with the increase in the radius $\epsilon$. To validate these claims, we re-ran the same experiment devised above while allowing for a larger $\epsilon$ of $0.015$. It is clear from Figure 3(c) that the robustness range of the policy generated by $\text{W}\text{R}^{2}\text{L}$ does increase with the increase in the ball's radius. 

These results were also verified in two additional benchmarks (i.e., the Walker and HalfCheetah). Here, again our results demonstrate that when 2 dimensional changes are considered, our method outperforms state-of-the-art significantly. We also arrive at the same conclusion that if $\epsilon$ increases so does the robustness range. For instance, a robust policy trained with an $\epsilon$ of 0.03 achieves high average test returns on a broader range of HalfCheetahs compares to that with an $\epsilon = 0.005$, see Figures 4 (h) and 4 (i). 

This, in turn, takes us to the following conclusion: 

\begin{tcolorbox}[enhanced, arc=0pt,outer arc=0pt, colback=white, colframe=black, drop shadow={black,opacity=1}]
\textbf{Experimental Conclusion II:} From the above, we conclude that $\text{W}\text{R}^{2}\text{L}$ outperforms others when two-dimensional simulator variations are considered, and that robustness increase with $\epsilon$.
\end{tcolorbox}

\paragraph{Results with High-Dimensional Model Variation:}  Though results above demonstrate robustness, an argument against a min-max objective can be made especially when only considering low-dimensional changes in the simulator. Namely, one can argue the need for such an objective as opposed to simply sampling a set of systems and determining policies performing-well on average similar to the approach proposed in~\citep{rajeswaran2016epopt}.

\begin{figure}[h!]
     \centering
     \begin{subfigure}[b]{0.32\textwidth}
         \centering
         \includegraphics[width=\textwidth]{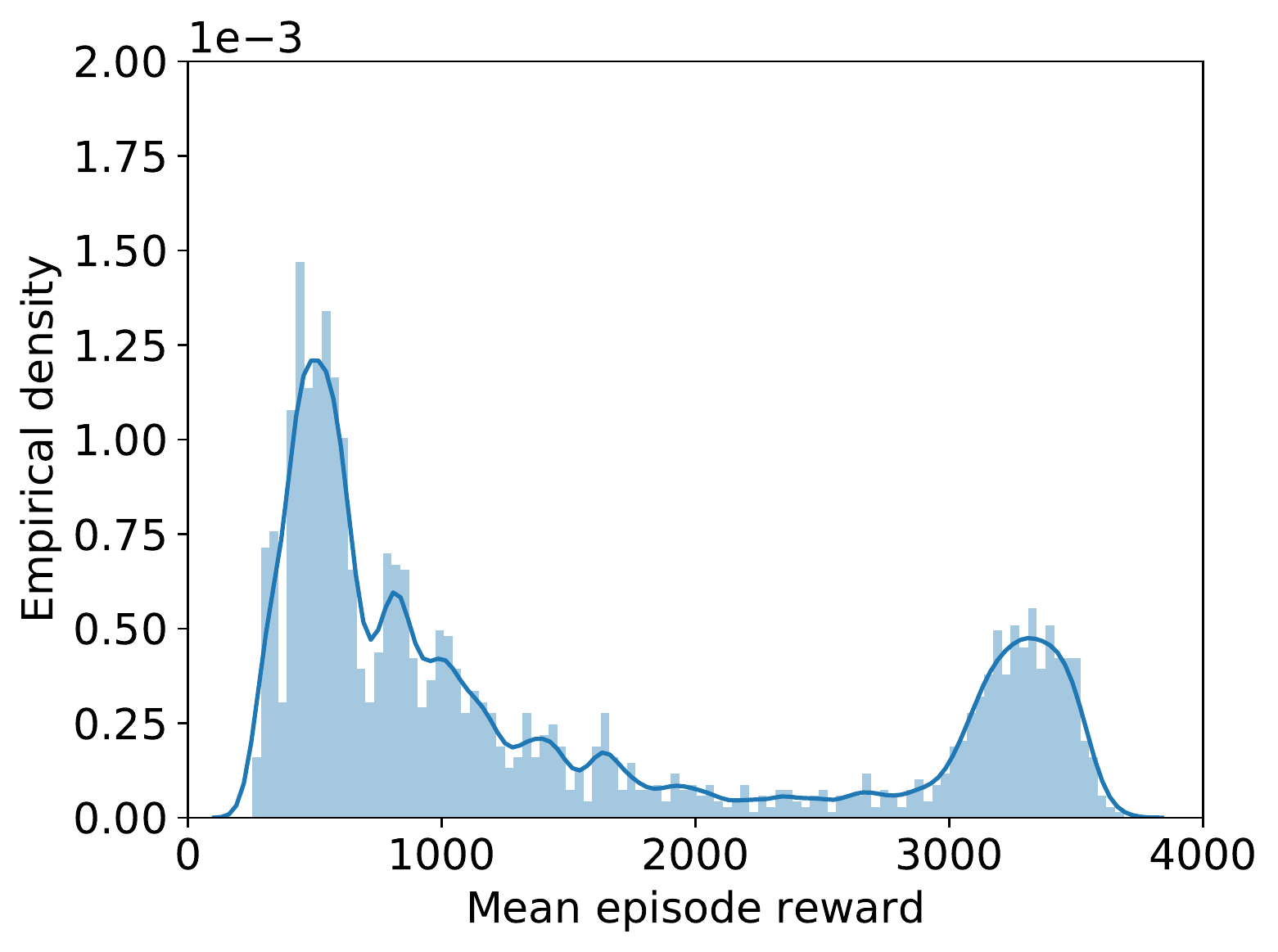}
         \caption{PPO High-D Var. - HP}
     \end{subfigure}
     \hfill
     \begin{subfigure}[b]{0.32\textwidth}
         \centering
         \includegraphics[width=\textwidth]{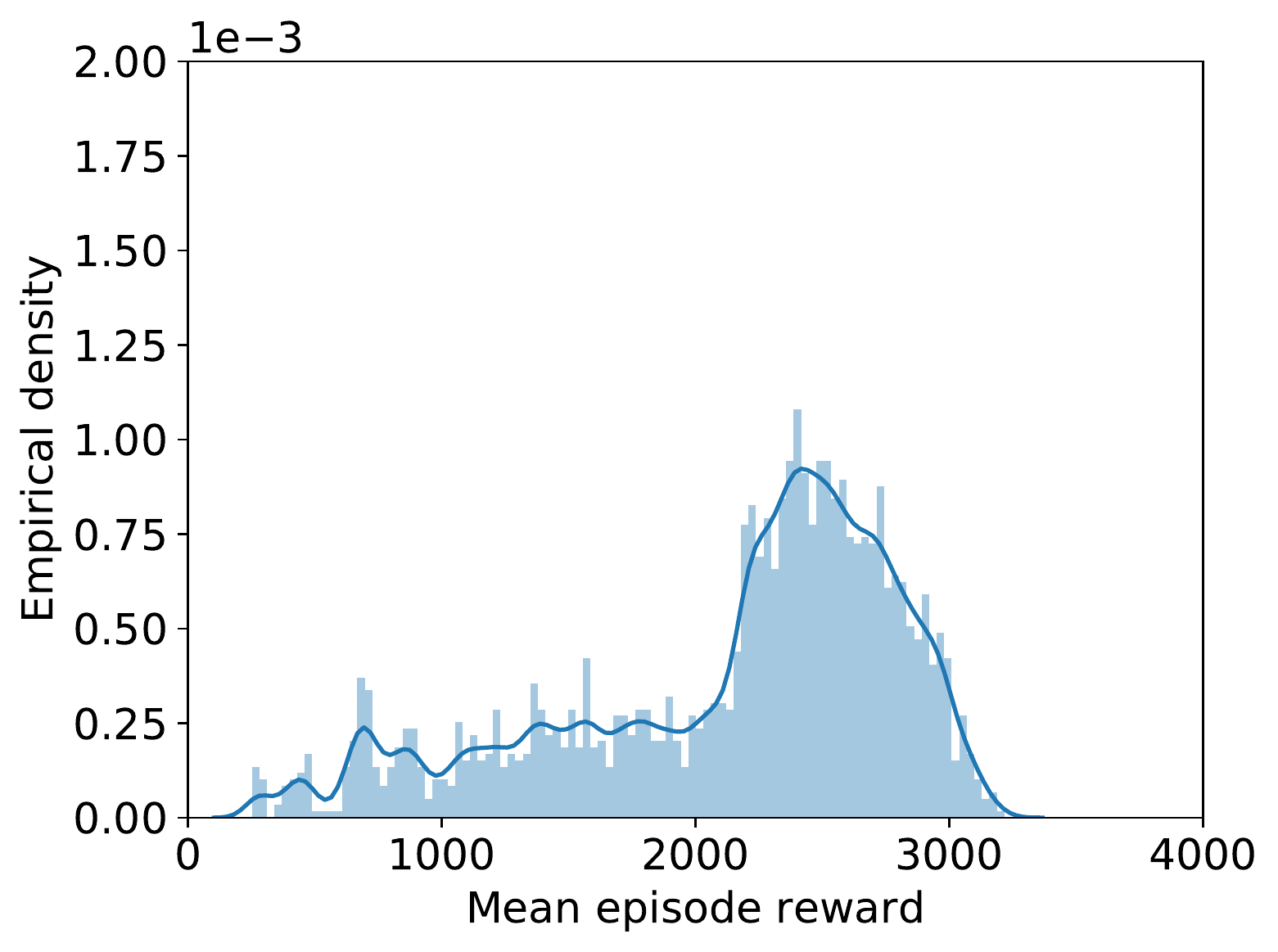}
         \caption{Train Low Test High - HP}
     \end{subfigure}
     \hfill
     \begin{subfigure}[b]{0.32\textwidth}
         \centering
         \includegraphics[width=\textwidth]{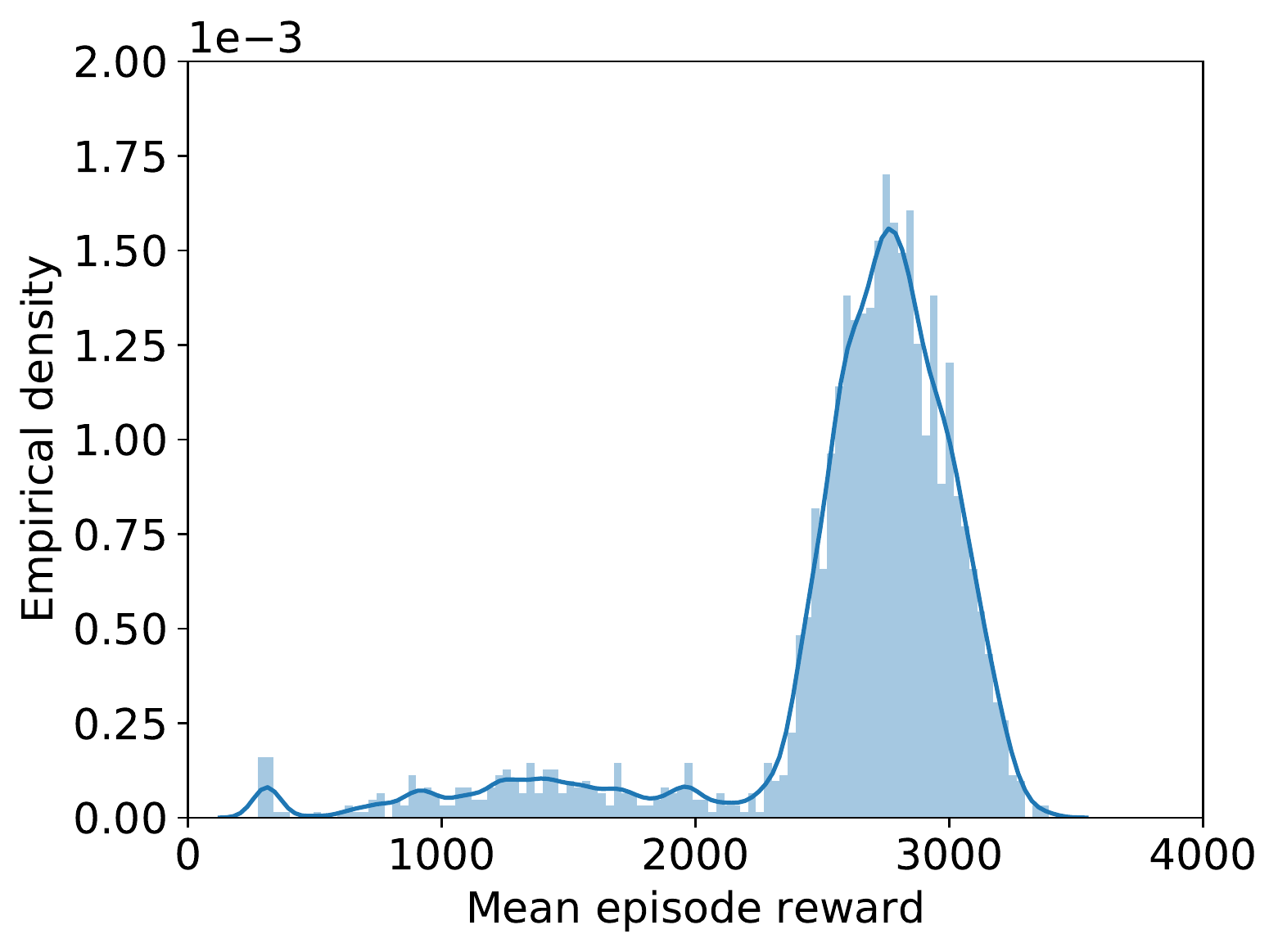}
         \caption{Train High Test High - HP}
     \end{subfigure}
    \begin{subfigure}[b]{0.32\textwidth}
         \centering
         \includegraphics[width=\textwidth]{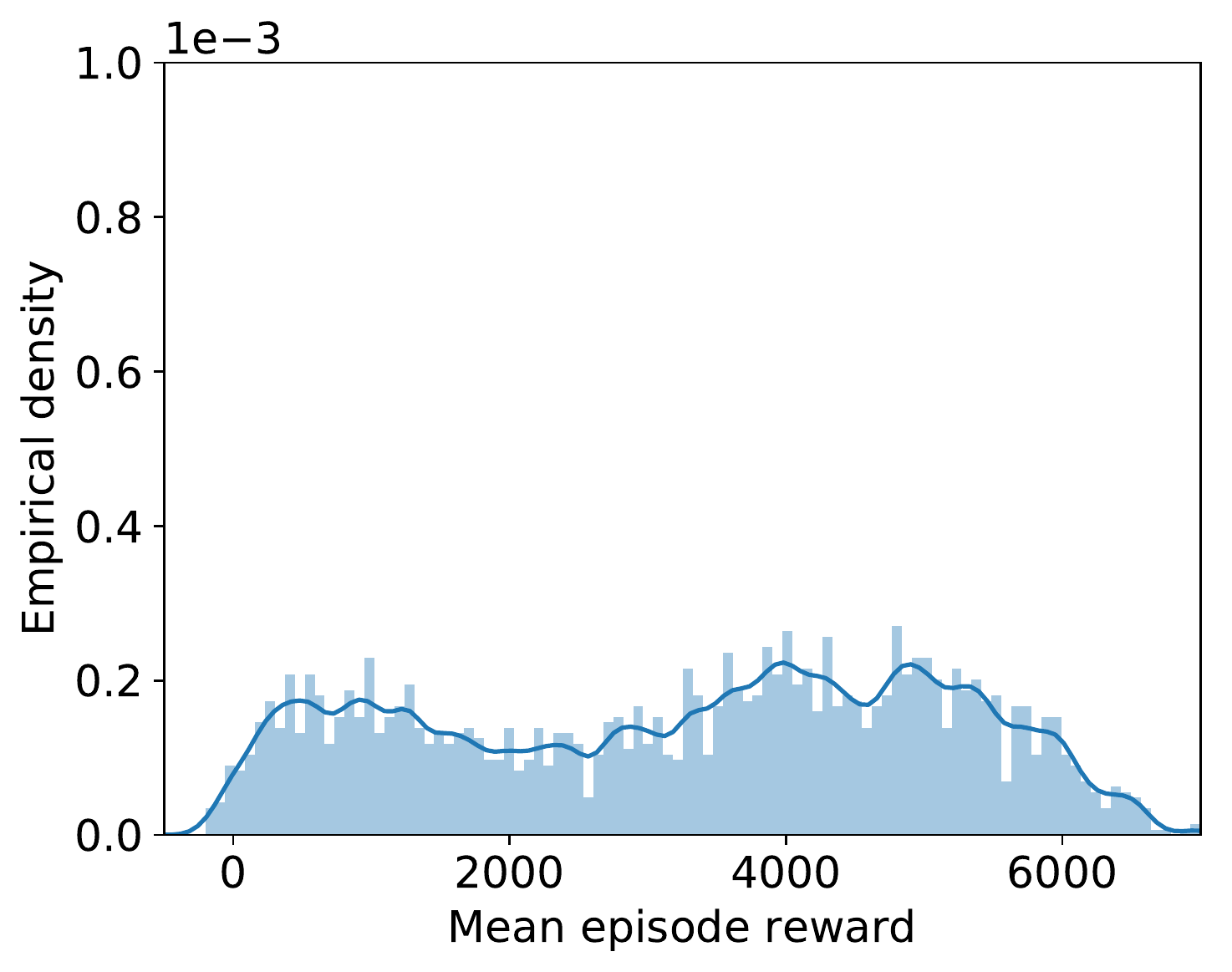}
     \caption{PPO High-D Var.- HC}
     \end{subfigure} 
     \hfill 
         \begin{subfigure}[b]{0.32\textwidth}
         \centering
         \includegraphics[width=\textwidth]{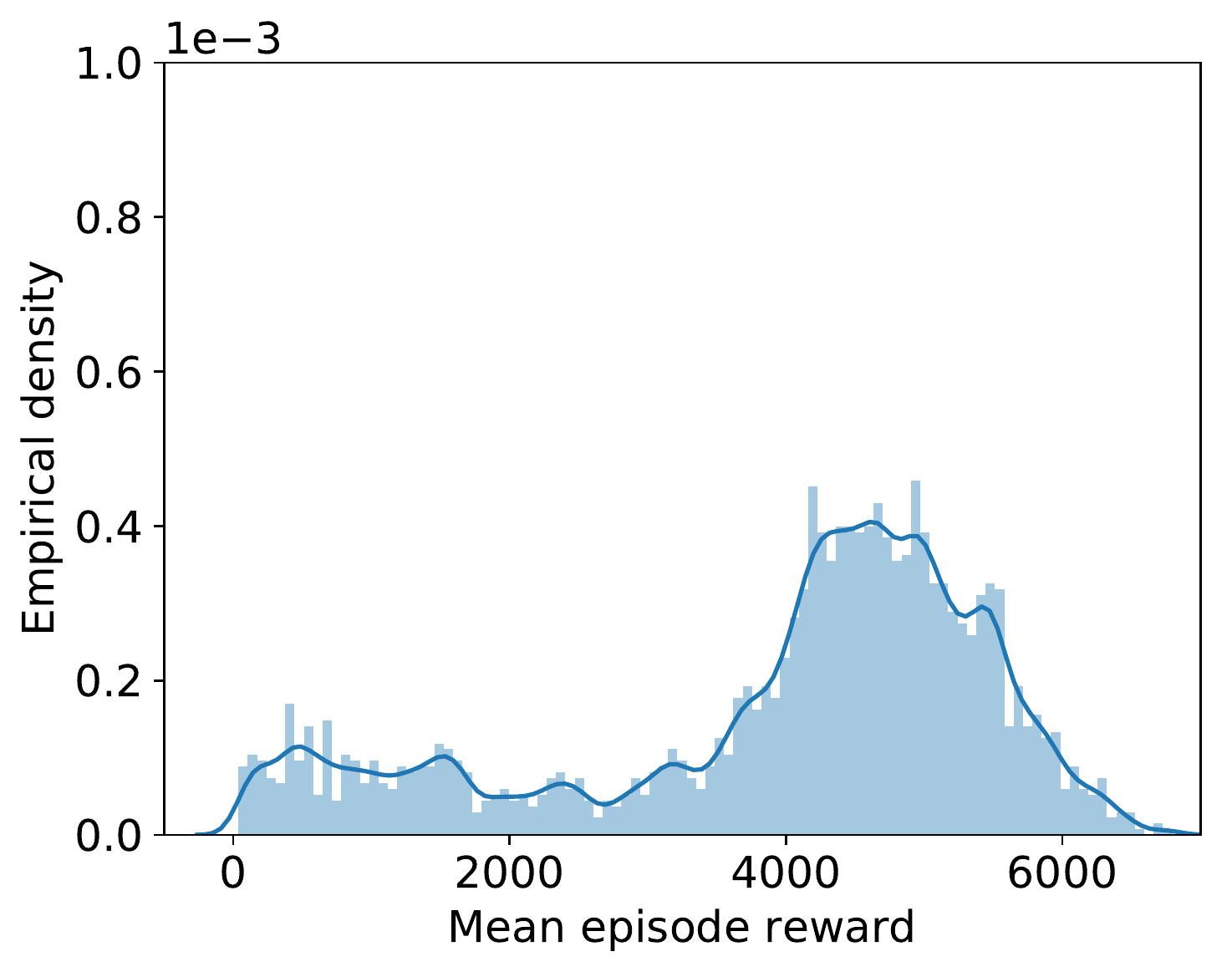}
     \caption{Train Low Test High - HC}
     \end{subfigure} 
     \hfill
    \begin{subfigure}[b]{0.32\textwidth}
         \centering
         \includegraphics[width=\textwidth]{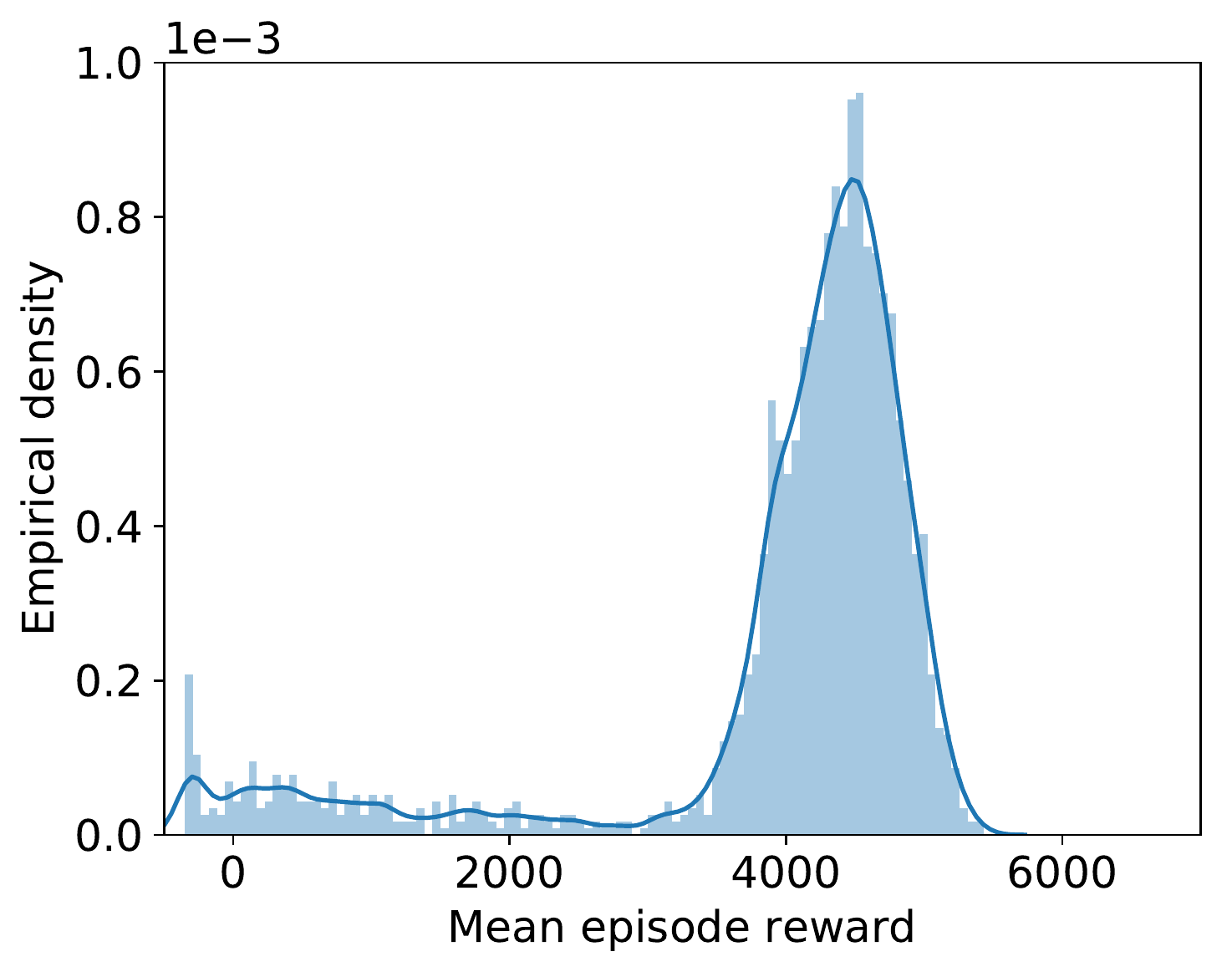}
     \caption{Train High Test High - HC}
     \end{subfigure} 
    \caption{Results evaluating performance when considering high-dimensional variations on the hopper (HP - top row) and HalfCheetah (HC - bottom row) environment. All figures show the empirical distribution of returns on 1,000 testing systems. Figure (a) demonstrates the robustness of PPO. Figure (b) reports empirical test returns of $\text{W}\text{R}^{2}\text{L}$'s policy trained on only two parameter changes (e.g., friction and density) of the environment but tested on systems with all high-dimensional dynamical parameters modified. Figure (c) trains and tests $\text{W}\text{R}^{2}\text{L}$ altering all dimensional parameters of the simulator. Clearly, our method exhibits robustness even if high-dimensional variations were considered.}
    \label{fig:exp_11d}
\end{figure}

A counter-argument to the above is that a gradient-based optimisation scheme is more efficient than a sampling-based one when high-dimensional changes are considered. In other words, a sampling procedure is hardly applicable when more than a few parameters are altered, while $\text{W}\text{R}^{2}\text{L}$ can remain suitable. To assess these claims, we conducted two additional experiments on the Hopper and HalfCheetah benchmarks. In the first, we trained robustly while changing friction and torso densities, and tested on 1000 systems generated by varying all 11 dimensions of the Hopper dynamics, and 21 dimensions of the HalfCheetah system. Results reported in Figures~\ref{fig:exp_11d}(b) and (e) demonstrate that the empirical densities of the average test returns are mostly centered around 3000 for the Hopper, and around 4500 for the Cheetah, which improves that of PPO (Figures~\ref{fig:exp_11d}(a) and (d)) with return masses mostly accumulated at around 1000 in the case of the Hopper and almost equally distributed when considering HalfCheetah.

Such improvements, however, can be an artifact of the careful choice of the low-dimensional degrees of freedom allowed to be modified during Phase I of Algorithm~\ref{Algo:Main}. To get further insights,
Figures~\ref{fig:exp_11d}(c) and (f) demonstrate the effectiveness of our method trained and tested while allowing to tune all 11 dimensional parameters of the Hopper simulator, and the 21 dimensions of the HalfCheetah. Indeed, our results are in accordance with these of the previous experiment depicting that most of the test returns' mass remains around 3000 for the Hopper, and improves to accumulate around 4500 for the HalfCheetah. Interestingly, however, our algorithm is now capable of acquiring higher returns on all systems\footnote{Please note that we attempted to compare against~\cite{rajeswaran2016epopt}. Due to the lack of open-source code, we were not able to regenerate their results. } since it is allowed to alter all parameters defining the simulator. As such, we conclude:

\begin{tcolorbox}[enhanced, arc=0pt,outer arc=0pt, colback=white, colframe=black, drop shadow={black,opacity=1}]
\textbf{Experimental Conclusion III:} From the above, we conclude that $\text{W}\text{R}^{2}\text{L}$ outperforms others when high-dimensional simulator variations are considered.
\end{tcolorbox}

\section{Conclusion \& Future Work}\label{Sec:Conc_Ftr_work}
In this paper, we proposed a novel robust reinforcement learning algorithm capable of outperforming others in terms of test returns on unseen dynamical systems. Our algorithm formalises a new min-max objective with Wasserstein constraints for policies generalising across varying domains, and considers a zero-order method for scalable solutions.  Empirically, we demonstrated superior performance against state-of-the-art from both standard and robust reinforcement learning on low and high-dimensional MuJuCo environments. 

There are a lot of interesting directions we plan to target in the future. First, we aim to consider robustness in terms of other components of MDPs, e.g., state representations, reward functions, and others. Second, we will implement $\text{W}\text{R}^{2}\text{L}$ on real hardware, considering sim-to-real experiments. 

\section{Acknowledgements}
We are grateful to Jan Peters and Andreas Krause for the interesting discussions that helped better-shape this paper. Moreover, we would like to thank each of Rasul Tutunov, and Aivar Sootla for guiding us through discussion on optimisation and control theory. Finally, we thank Jun Yao, Liu Wulong, Chen Zhitang, Jia Zheng, and Zhengguo Li for helping us improve this work with inputs about real-world considerations.

\appendix
\section{Deep Deterministic Policy Gradients Results}\label{App:InvBar}
As mentioned in the experiments section of the main paper, we refrained from presenting results involving deep deterministic policy gradients (DDPG) due to its lack in robustness even on simple systems, such as the CartPole. 
\begin{figure}[h!]
 \centering
 \includegraphics[scale=.6, trim = {0em 0em 0em 0em}, clip]{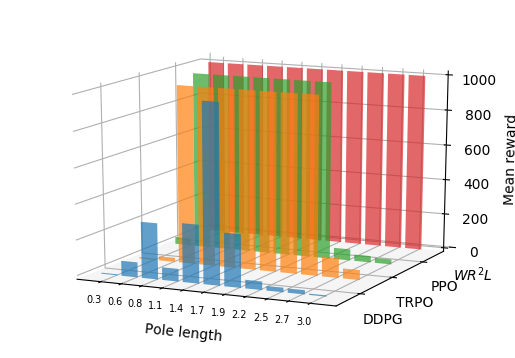}
\caption{Robustness results on the inverted pendulum demonstrating that our method outperforms state-of-the-art in terms of average test returns and that DDPG lacks in robustness performance.}
\label{fig:myfigf}
\end{figure}

Figure~\ref{fig:myfigf} depicts these results showing that DDPG lacks robustness even when minor variations in the pole length are introduced. TRPO and PPO, on the other hand, demonstrate an acceptable performance retaining a test return of 1,000 across a broad range of pole lengths variations. 

\section{Derivation of the Closed Form Solution}
In Section~\ref{Sec:Sol} we presented a closed form solution to the following optimisation problem: 
\begin{equation*}
    \min_{\bm{\phi}} \nabla_{\bm{\phi}} \mathbb{E}_{\bm{\tau}\sim p_{\bm{\theta}}^{\bm{\phi}}(\bm{\tau})}\left[\mathcal{R}_{\text{total}}(\bm{\tau})\right] \Bigg|^{\mathsf{T}}_{\bm{\theta}^{[k]}, \bm{\phi}^{[j]}}(\bm{\phi} - \bm{\phi}^{[j]}) \ \ \text{s.t.} \ \ \frac{1}{2} (\bm{\phi} - \bm{\phi}_{0})^{\mathsf{T}}\bm{H}_{0}(\bm{\phi} - \bm{\phi}_{0})\leq \epsilon, 
\end{equation*}
which took the form of: 
\begin{align*}
    \bm{\phi}^{[j+1]} = \bm{\phi}_{0} - \sqrt{\frac{2\epsilon}{\bm{g}^{[k,j] \mathsf{T}}\bm{H}_{0}^{-1}\bm{g}^{[k,j]}}} \bm{H}_{0}^{-1}\bm{g}^{[k,j]}.
\end{align*}
In this section of the appendix, we derive such an update rule from first principles. We commence transforming the constraint optimisation problem into an unconstrained one using the Lagrangian: 
\begin{equation*}
    \mathcal{L}(\bm{\phi}, \bm{\lambda}) = \bm{g}^{[k,j], \mathsf{T}}\left(\bm{\phi}-\bm{\phi}^{[j]}\right)  + \bm{\lambda}\left[\frac{1}{2} (\bm{\phi} - \bm{\phi}_{0})^{\mathsf{T}}\bm{H}_{0}(\bm{\phi} - \bm{\phi}_{0}) -\epsilon \right],
\end{equation*}
where $\bm{\lambda}$ is a Lagrange multiplier, and $\bm{g}^{[k,j]} = \nabla_{\bm{\phi}} \mathbb{E}_{\bm{\tau}\sim p_{\bm{\theta}}^{\bm{\phi}}(\bm{\tau})}\left[\mathcal{R}_{\text{total}}(\bm{\tau})\right]\Bigg|^{\mathsf{T}}_{\bm{\theta}^{[k]}, \bm{\phi}^{[j]}}$. 

Deriving the Lagrangian with respect to the primal parameters $\bm{\phi}$, we write: 
\begin{align}
\label{Eq:BlaBlo}
    \nabla_{\bm{\phi}} \mathcal{L}(\bm{\phi}, \bm{\lambda}) = \bm{g}^{[k,j] \mathsf{T}} + \bm{\lambda}(\bm{\phi} - \bm{\phi}_{0})^{\mathsf{T}}\bm{H}_{0}. 
\end{align}
Setting Equation~\ref{Eq:BlaBlo} to zero and solving for primal parameters, we attain: 
\begin{equation*}
    \bm{\phi} = \bm{\phi}_{0} - \frac{1}{\bm{\lambda}}\bm{H}_{0}^{-1}\bm{g}^{[k,j]}. 
\end{equation*}
Plugging $\bm{\phi}$ back into the equation representing the constraints, we derive: 
\begin{align*}
    \left(\bm{\phi}_{0} - \frac{1}{\bm{\lambda}}\bm{H}_{0}^{-1}\bm{g}^{[k,j]} - \bm{\phi}_{0}\right)^{\mathsf{T}} \bm{H}_{0}\left(\bm{\phi}_{0} - \frac{1}{\bm{\lambda}}\bm{H}_{0}^{-1}\bm{g}^{[k,j]} - \bm{\phi}_{0}\right)=2\epsilon \implies \bm{\lambda}^{2} = \frac{1}{2\epsilon} \bm{g}^{[k,j]\mathsf{T}}\bm{H}_{0}^{-1}\bm{g}^{[k,j]}.
\end{align*}
It is easy to see that with the positive solution for $\bm{\lambda}$, the Karush–Kuhn–Tucker (KKT) conditions are satisfied. Since the objective and constraint are both convex, the KKT conditions are sufficient and necessary for optimality, thus finalising our derivation.

\bibliographystyle{apalike}
\bibliography{ref}

\end{document}